\documentclass{article}
\pdfoutput=1

\usepackage[final, nonatbib]{neurips_2019}
\usepackage{amsfonts}
\usepackage{amsmath}
\usepackage{amsthm}
\usepackage{amssymb}
\usepackage{graphicx}
\usepackage{caption}
\usepackage{subcaption}
\usepackage{xcolor}
\usepackage[utf8]{inputenc}
\usepackage{thmtools, thm-restate}
\usepackage[normalem]{ulem}

\newtheorem{theorem}{Theorem}[section]
\newtheorem{lemma}[theorem]{Lemma}
\newtheorem{definition}[theorem]{Definition}
\newtheorem{example}[theorem]{Example}
\newtheorem{proposition}[theorem]{Proposition}

\def\shownotes{1}  \ifnum\shownotes=0
\newcommand{\authnote}[2]{$\ll$\textsf{\footnotesize #1 notes: #2}$\gg$}
\else
\newcommand{\authnote}[2]{}
\fi
\newcommand{\tnote}[1]{{\color{blue}\authnote{Tengyu}{#1}}}

\newcommand{\pl}[1]{}
\newcommand{\tm}[1]{}
\newcommand{\ak}[1]{}
\newcommand{\pldel}[1]{}

\newcommand{\expect}[0]{\ensuremath{\mathop{\mathbb{E}}}}
\newcommand{\prob}[0]{\ensuremath{\mathbb{P}}}
\newcommand{\ourcal}[0]{the scaling-binning calibrator}
\newcommand{\Ourcal}[0]{The scaling-binning calibrator}

\newcommand{\Ourcalnothe}[0]{Scaling-binning calibrator}

\graphicspath{ {images/} }

\usepackage[numbers]{natbib}
\usepackage[utf8]{inputenc} 
\usepackage[T1]{fontenc}    
\usepackage{hyperref}       
\usepackage{url}            
\usepackage{booktabs}       
\usepackage{amsfonts}       
\usepackage{nicefrac}       
\usepackage{microtype}      

\DeclareMathOperator*{\argmax}{arg\,max}
\DeclareMathOperator*{\argmin}{arg\,min}

\newcommand{\bins}{\mathcal{B}}

\title{Verified Uncertainty Calibration}

\author{%
  Ananya Kumar, Percy Liang, Tengyu Ma \\
  Department of Computer Science\\
  Stanford University\\
  \texttt{\{ananya, pliang, tengyuma\}@cs.stanford.edu} \\
}

\begin{document}

\maketitle

\begin{abstract}
  Applications such as weather forecasting and personalized medicine demand models that output calibrated probability estimates---those representative of the true likelihood of a prediction. Most models are not calibrated out of the box but are recalibrated by post-processing model outputs. We find in this work that popular recalibration methods like Platt scaling and temperature scaling are (i) less calibrated than reported, and (ii) current techniques cannot estimate how miscalibrated they are. An alternative method, histogram binning, has measurable calibration error but is sample inefficient---it requires $O(B/\epsilon^2)$ samples, compared to $O(1/\epsilon^2)$ for scaling methods, where $B$ is the number of distinct probabilities the model can output. To get the best of both worlds, we introduce \emph{\ourcal{}}, which first fits a parametric function to reduce variance and then bins the function values to actually ensure calibration. This requires only $O(1/\epsilon^2 + B)$ samples.
  Next, we show that we can estimate a model's calibration error more accurately using an estimator from the meteorological community---or equivalently measure its calibration error with fewer samples ($O(\sqrt{B})$ instead of $O(B)$).
  We validate our approach with multiclass calibration experiments on CIFAR-10 and ImageNet, where we obtain a 35\% lower calibration error than histogram binning and, unlike scaling methods, guarantees on true calibration.
  In these experiments, we also estimate the calibration error and ECE more accurately than the commonly used plugin estimators.
  We implement all these methods in a Python library: \url{https://pypi.org/project/uncertainty-calibration}

\end{abstract}


\pl{when we say $O(...)$, we're actually hiding $\log B$ factors; probably should use $\tilde O$?}

\section{Introduction}

The probability that a system outputs for an event should reflect the true frequency of that event: if an automated diagnosis system says 1,000 patients have cancer with probability 0.1, approximately 100 of them should indeed have cancer.
In this case, we say the model is uncertainty calibrated. The importance of this notion of calibration has been emphasized in personalized medicine~\cite{jiang2012calibrating}, meteorological forecasting~\cite{murphy1973vector, murphy1977reliability, degroot1983forecasters,gneiting2005weather, brocker2009decomposition} and natural language processing applications~\cite{nguyen2015posterior, card2018calibration}.
As most modern machine learning models, such as neural networks, do not output calibrated probabilities out of the box~\cite{guo2017calibration, zadrozny2001calibrated, kuleshov2018accurate}, reseachers use recalibration methods that take the output of an uncalibrated model, and transform it into a calibrated probability.
\emph{Scaling} approaches for recalibration---Platt scaling~\cite{platt1999probabilistic}, isotonic regression~\cite{zadrozny2002transforming}, and temperature scaling~\cite{guo2017calibration}---are widely used and require very few samples, but do they actually produce calibrated probabilities?

\emph{We discover that these methods are less calibrated than reported.} Past work approximates a model's calibration error using a finite set of bins. We show that by using more bins, we can uncover a higher calibration error for models on CIFAR-10 and ImageNet.
We show that a fundamental limitation with approaches that output a continuous range of probabilities is that their true calibration error is unmeasurable with a finite number of bins (Example~\ref{ex:continuous-not-calibrated}).

An alternative approach, histogram binning~\cite{zadrozny2001calibrated}, outputs probabilities from a finite set.
Histogram binning can produce a model that is calibrated, and unlike scaling methods we can measure its calibration error, but it is sample inefficient.
In particular, the number of samples required to calibrate scales linearly with the number of distinct probabilities the model can output, $B$~\cite{naeini2014binary}, which can be large particularly in the multiclass setting where $B$ typically scales with the number of classes.
Recalibration sample efficiency is crucial---we often want to recalibrate our models in the presence of domain shift~\cite{hendrycks2019anomaly} or recalibrate a model trained on simulated data, and may have access to only a small labeled dataset from the target domain.

To get the sample efficiency of Platt scaling and the verification guarantees of histogram binning, \emph{we propose \ourcal{}} (Figure~\ref{fig:var_red_binning}).
Like scaling methods, we fit a simple function $g \in \mathcal{G}$ to the recalibration dataset.
We then bin the input space so that an equal number of inputs land in each bin.
In each bin, we output the average of the $g$ values in that bin---these are the gray circles in Figure~\ref{fig:var_red_binning}.
In contrast, histogram binning outputs the average of the label values in each bin (Figure~\ref{fig:hist_binning}).
The motivation behind our method is that the $g$ values in each bin are in a narrower range than the label values, so when we take the average we incur lower estimation error.
If $\mathcal{G}$ is well chosen\pl{what does this mean?}, our method requires $O(\frac{1}{\epsilon^2} + B)$ samples to achieve calibration error $\epsilon$ instead of $O(\frac{B}{\epsilon^2})$ samples for histogram binning, where $B$ is the number of model outputs (Theorem~\ref{thm:final-calib}). Note that in prior work, binning the outputs of a function was used for evaluation and without any guarantees, whereas in our case it is used for the method itself, and we show improved sample complexity.

\begin{figure}
     \centering
     \begin{subfigure}[b]{0.32\textwidth}
         \centering
         \includegraphics[width=\textwidth]{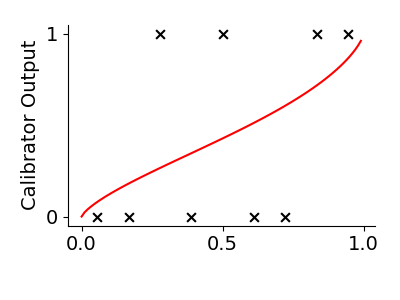}
         \caption{Platt scaling}
         \label{fig:platt_scaling}
     \end{subfigure}
     \hfill
     \begin{subfigure}[b]{0.32\textwidth}
         \centering
         \includegraphics[width=\textwidth]{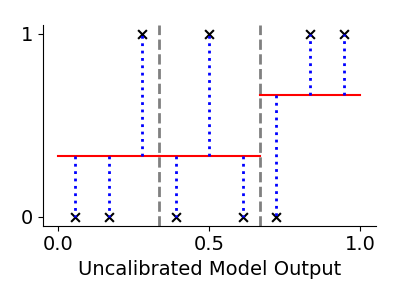}
         \caption{Histogram binning}
         \label{fig:hist_binning}
     \end{subfigure}
     \hfill
     \begin{subfigure}[b]{0.32\textwidth}
         \centering
         \includegraphics[width=\textwidth]{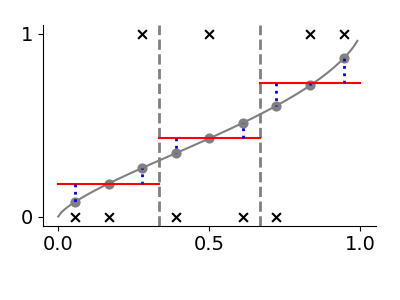}
         \caption{\Ourcalnothe{}}
         \label{fig:var_red_binning}
     \end{subfigure}
        \caption{
        Visualization of the three recalibration approaches.
        The black crosses are the ground truth labels, and the red lines are the output of the recalibration methods.
        Platt Scaling (Figure~\ref{fig:platt_scaling}) fits a function to the recalibration data, but its calibration error is not measurable.
        Histogram binning (Figure~\ref{fig:hist_binning}) outputs the average label in each bin.
        \Ourcal{} (Figure~\ref{fig:var_red_binning}) fits a function $g \in \mathcal{G}$ to the recalibration data and then \emph{takes the average of the function values (the gray circles)} in each bin.
        The function values have lower variance than the labels, as visualized by the blue dotted lines, which is why our approach has lower variance.
        }
        \label{fig:variance_reduced_illustration}
\end{figure}

We run multiclass calibration experiments on CIFAR-10~\cite{krizhevsky2009learningmultiple} and ImageNet~\cite{deng2009imagenet}.
\Ourcal{} achieves a lower calibration error than histogram binning, while allowing us to measure the true calibration error unlike for scaling methods.
We get a \emph{35\% lower calibration error on CIFAR-10} and a \emph{5x lower calibration error on ImageNet} than histogram binning for $B = 100$.

Finally, we show how to estimate the calibration error of models more accurately. Prior work in machine learning~\cite{nguyen2015posterior, guo2017calibration, hendrycks2019anomaly, kuleshov2015calibrated, hendrycks2019pretraining} directly estimates each term\pl{haven't introduced term, so doesn't make sense; choose another way to explain} in the calibration error from samples (Definition~\ref{dfn:plugin-estimator}). The sample complexity of this plugin estimator scales linearly with $B$. A debiased estimator introduced in the meteorological literature~\cite{brocker2012empirical, ferro2012bias} reduces the bias of the plugin estimator; \emph{we prove that it achieves sample complexity that scales with $\sqrt{B}$} by leveraging error cancellations across bins. Experiments on CIFAR-10 and ImageNet confirm that the debiased estimator measures the calibration error more accurately.


\section{Setup and background}
\label{sec:formulation}

\newcommand{\ce}[0]{\ensuremath{\textup{CE}}}
\newcommand{\lpce}[0]{\ensuremath{\ell_p\textup{-CE}}}
\newcommand{\ltwoce}[0]{\ensuremath{\ell_2\textup{-CE}}}
\newcommand{\lsquared}[0]{\ensuremath{L^2}}
\newcommand{\squaredce}[0]{\ensuremath{L^2\textup{-CE}}}
\newcommand{\topce}[0]{\ensuremath{\textup{TCE}}}
\newcommand{\margsquaredce}[0]{\ensuremath{\textup{MCE}}}
\newcommand{\mse}[0]{\ensuremath{\textup{MSE}}}

\subsection{Binary classification}

Let $\mathcal{X}$ be the input space and $\mathcal{Y}$ be the label space where $\mathcal{Y} = \{0, 1\}$ for binary classification.
Let $X \in \mathcal{X}$ and $Y \in \mathcal{Y}$ be random variables denoting the input and label, given by an unknown joint distribution \pl{use mathbb}$P$. As usual, expectations are taken over all random variables.

Suppose we have a model $f : \mathcal{X} \to [0, 1]$ where the (possibly uncalibrated) output of the model represents the model's confidence that the label is 1. The calibration error examines the difference between the model's probability and the true probability given the model's output:

\begin{definition}[Calibration error]
The calibration error of $f : \mathcal{X} \to [0, 1]$ is given by:
\begin{align}
\ce(f) = \Big(\expect\big[ \left|f(X) - \expect[Y \mid f(X)] \right|^2 \big] \Big)^{1/2}
\end{align}
\end{definition}

If $\ce(f) = 0$ then $f$ is perfectly calibrated. This notion of calibration error is most commonly used~\cite{murphy1973vector,murphy1977reliability,degroot1983forecasters, nguyen2015posterior, hendrycks2019anomaly, kuleshov2015calibrated, hendrycks2019pretraining, brocker2012empirical}. Replacing the $2$s in the above definition by $p \geq 1$ we get the $\ell_p$ calibration error---the $\ell_1$ and $\ell_{\infty}$ calibration errors are also used in the literature~\cite{guo2017calibration, naeini2015obtaining, nixon2019calibration}. In addition to $\ce{}$, we also deal with the $\ell_1$ calibration error (known as ECE) in Sections~\ref{sec:challenges-measuring} and~\ref{sec:verifying_calibration}.



Calibration alone is not sufficient: consider an image dataset containing $50\%$ dogs and $50\%$ cats.
If $f$ outputs $0.5$ on all inputs, $f$ is calibrated but not very useful.
We often also wish to minimize the mean-squared error---also known as the Brier score---subject to a calibration budget~\cite{gneiting2005weather, gneiting2007probabilistic}.

\begin{definition}
The mean-squared error of $f : \mathcal{X} \to [0, 1]$ is given by $\textup{MSE}(f) = \mathbb{E}[(f(X) - Y)^2]$.
\end{definition}

Note that $\textup{MSE}$ and $\ce{}$ are not orthogonal\pl{independent?} and $\mbox{MSE} = 0$ implies perfect calibration; in fact the MSE is the sum of the squared calibration error and a ``sharpness'' term~\cite{murphy1973vector,degroot1983forecasters, kuleshov2015calibrated}.

\subsection{Multiclass classification}

While calibration in binary classification is \pl{well-defined?}well-studied,
it's less clear what to do for multiclass\pl{make less casual}, where multiple definitions abound, differing in their strengths. In the multiclass setting, $\mathcal{Y} = [K] = \{1, \dots, K\}$ and $f : \mathcal{X} \to [0, 1]^K$ outputs a confidence measure for each class in $[K]$\pl{probability distribution over the $K$ classes [have to sum to 1, right?]}.

\begin{definition}[Top-label calibration error]
The top-label calibration error examines the difference between the model's probability for its top prediction and the true probability of that prediction given the model's output:
\begin{align}
\topce(f) = \Big( \expect\Big[ \Big( \prob\big(Y = \argmax_{j \in [K]} f(X)_j \mid \max_{j \in [K]} f(X)_j\big) - \max_{j \in [K]} f(X)_j \Big)^2 \Big] \Big)^{1/2}
\end{align}
\end{definition}

We would often like the model to be calibrated on less likely predictions as well---imagine that a medical diagnosis system says there is a $50\%$ chance a patient has a benign tumor, a $10\%$ chance she has an aggressive form of cancer, and a $40\%$ chance she has one of a long list of other conditions. We would like the model to be calibrated on all of these predictions so we define the marginal calibration error which examines, \emph{for each class}, the difference between the model's probability and the true probability of that class given the model's output.

\begin{definition}[Marginal calibration error]
\label{dfn:marginal-ce}
  Let $w_k \in [0, 1]$ denote how important calibrating class $k$ is, where $w_k = 1/k$ if all classes are equally important. The marginal calibration error is:
\begin{align}
\margsquaredce(f) = \Big( \sum_{k = 1}^K w_k \mathbb{E}\big[ (f(X)_k - \prob(Y = k \mid f(X)_k))^2 \big] \Big)^{1/2}
\end{align}
\end{definition}

Prior works~\cite{guo2017calibration, hendrycks2019anomaly, hendrycks2019pretraining} propose methods for multiclass calibration but only measure top-label calibration---\cite{nixon2019calibration} and concurrent work to ours~\cite{kull2019temperature} define similar per-class calibration metrics where temperature scaling~\cite{guo2017calibration} is worse than vector scaling despite having better top-label calibration.


For notational simplicity, our theory focuses on the binary classification setting. We can transform top-label calibration into a binary calibration problem---the model outputs a probability corresponding to its top prediction, and the label represents whether the model gets it correct or not. Marginal calibration can be transformed into $K$ one-vs-all binary calibration problems where for each $k \in [K]$ the model outputs the probability associated with the $k$-th class, and the label represents whether the correct class is $k$~\cite{zadrozny2002transforming}. We consider both top-label calibration and marginal calibration in our experiments.
Other notions of multiclass calibration include joint calibration (which requires the entire probability \emph{vector} to be calibrated)~\cite{murphy1973vector, brocker2009decomposition} and event-pooled calibration~\cite{kuleshov2015calibrated}.

\subsection{Recalibration}

Since most machine learning models do not output calibrated probabilities out of the box~\cite{guo2017calibration, zadrozny2001calibrated} recalibration methods take the output of an uncalibrated model, and transform it into a calibrated probability. That is, given a trained model $f: \mathcal{X} \to [0, 1]$, let $Z = f(X)$. We are given recalibration data $T = \{ (z_i, y_i) \}_{i=1}^n$ independently sampled from \pl{mathbb} $P(Z, Y)$, and we wish to learn a recalibrator $g : [0, 1] \to [0, 1]$ such that $g \circ f$ is well-calibrated.

\emph{Scaling methods}, for example Platt scaling~\cite{platt1999probabilistic}, output a function $g = \argmin_{g \in \mathcal{G}} \sum_{(z, y) \in T} \ell(g(z), y)$, where $\mathcal{G}$ is a model family, $g \in \mathcal{G}$ is differentiable, and $\ell$ is a loss function, for example the log-loss or mean-squared error. The advantage of such methods is that they converge very quickly since they only fit a small number of parameters.

\emph{Histogram binning} first constructs a set of bins (intervals) that partitions $[0, 1]$, formalized below.

\begin{definition}[Binning schemes]
A binning scheme $\mathcal{B}$ of size $B$ is a set of $B$ intervals $I_1, \dots, I_B$ that partitions $[0, 1]$. Given $z \in [0, 1]$, let $\beta(z) = j$, where $j$ is the interval that $z$ lands in ($z \in I_j$).
\end{definition}

The bins are typically chosen such that either $I_1 = [0, \frac{1}{B}], I_2 = (\frac{1}{B}, \frac{2}{B}], \dots, I_B = (\frac{B-1}{B}, 1]$ (equal width binning)~\cite{guo2017calibration} or so that each bin contains an equal number of $z_i$ values in the recalibration data (uniform mass binning)~\cite{zadrozny2001calibrated}. Histogram binning then outputs the average $y_i$ value in each bin.

\section{Is Platt scaling calibrated?}
\label{sec:challenges-measuring}

In this section, we show that methods like Platt scaling and temperature scaling are (i) less calibrated than reported and (ii) it is difficult\pl{impossible in general [difficult might mean that we just didn't try hard enough]} to tell how miscalibrated they are. That is we show, both theoretically and with experiments on CIFAR-10 and ImageNet, why the calibration error of models that output a continuous range of values is \emph{underestimated}.
We defer proofs to Appendix~\ref{sec:appendix-platt-not-calibrated}.

The key to estimating the calibration error is estimating the conditional expectation $\expect[Y \mid f(X)]$.  If $f(X)$ is continuous, without smoothness assumptions on $\expect[Y \mid f(X)]$ (that cannot be verified in practice), this is impossible. This is analogous to the difficulty of measuring the mutual information between two continuous signals~\cite{paninski2003entropy}.

To approximate the calibration error, prior work bins the output of $f$ into $B$ intervals.
The calibration error in each bin is estimated as the difference between the average value of $f(X)$ and $Y$ in that bin.
Note that the binning here is for evaluation only, whereas in histogram binning, it is used for the recalibration method itself.
We formalize the notion of this binned calibration error below.

\begin{definition}
The binned version of $f$ outputs the average value of $f$ in each bin $I_j$:
\begin{align}
  f_{\mathcal{B}}(x) = \expect[f(X) \mid f(X) \in I_j] \quad\quad\quad \mbox{where }x \in I_j
\end{align} 
\end{definition}

Given $\bins{}$, the binned calibration error of $f$ is simply the calibration error of $f_{\bins{}}$.
A simple example shows that using binning to estimate the calibration error can severely underestimate the true calibration error.

\newcommand{\continuousNotCalibratedText}{
  For any binning scheme $\bins{}$, and continuous bijective function $f : [0, 1] \to [0, 1]$, there exists a distribution \pl{mathbb}$P$ over $\mathcal{X}, \mathcal{Y}$ s.t. $\ce(f_{\bins{}}) = 0$ but $\ce(f) \geq 0.49$.
Note that for all $f$, $0 \leq \ce(f) \leq 1$.
}

\begin{example}
\label{ex:continuous-not-calibrated}
\continuousNotCalibratedText{}
\end{example}

\newtheorem*{continuousNotCalibrated}{Restatement of Example~\ref{ex:continuous-not-calibrated}}

The intuition of the construction is that in each interval $I_j$ in $\bins{}$, the model could underestimate the true probability $\expect[Y \mid f(X)]$ half the time, and overestimate the probability half the time. So if we average over the entire bin the model appears to be calibrated, even though it is very uncalibrated. The formal proof is in Appendix~\ref{sec:appendix-platt-not-calibrated}, and holds for arbitrary $\ell_p$ calibration errors including the ECE.

Next, we show that given a function $f$, its binned version always has lower calibration error. The proof, in Appendix~\ref{sec:appendix-platt-not-calibrated}, is by Jensen's inequality. Intuitively, averaging a model's prediction within a bin allows errors at different parts of the bin to cancel out with each other. This result is similar to Theorem 2 in recent work~\cite{vaicenavicius2019calibration}, and holds for arbitrary $\ell_p$ calibration errors including the ECE.

\newcommand{\binningLowerBoundText}{
  Given any binning scheme $\bins{}$ and model $f : \mathcal{X} \to [0, 1]$, we have:
\[  \ce(f_{\bins{}}) \leq \ce(f). \]
\pl{number the equations}
}

\begin{proposition}[Binning underestimates error]
\label{prop:bin_low_bound}
\binningLowerBoundText{}
\end{proposition}

\newtheorem*{binningLowerBound}{Restatement of Proposition~\ref{prop:bin_low_bound}}

\subsection{Experiments}

Our experiments on ImageNet and CIFAR-10 suggest that previous work reports numbers which are lower than the actual calibration error of their models. Recall that binning lower bounds the calibration error. We cannot compute the actual calibration error but if we use a `finer' set of bins then we get a tighter lower bound on the calibration error.

As in~\cite{guo2017calibration}, our model's objective was to output the top predicted class and a confidence score associated with the prediction. For ImageNet, we started with a trained VGG16 model with an accuracy of 64.3\%. We split the validation set into 3 sets of sizes $(20000, 5000, 25000)$. We used the first set of data to recalibrate the model using Platt scaling, the second to select the binning scheme $\mathcal{B}$ so that each bin contains an equal number of points, and the third to measure the binned calibration error \pl{that's a huge number of samples! maybe have a remark that says you normally wouldn't do this if you trust your calibrator}. We calculated $90\%$ confidence intervals for the binned calibration error using 1,000 bootstrap resamples and performed the same experiment with varying numbers of bins.

Figure~\ref{fig:imagenet_lower_bound} shows that as we increase the number of bins on ImageNet, the measured calibration error is higher and this is statistically significant\pl{make this sound less simplistic}. For example, if we use 15 bins as in~\cite{guo2017calibration}, we would think the calibration error is around 0.02 when in reality the calibration error is at least twice as high. Figure~\ref{fig:cifar_10_lower_bound} shows similar findings for CIFAR-10, and in Appendix~\ref{sec:appendix_platt_experiments} we show that our findings hold even if we use the $\ell_1$ calibration error (ECE) and alternative binning strategies.

\begin{figure}
     \centering
     \begin{subfigure}[b]{0.4\textwidth}
         \centering
         \includegraphics[width=\textwidth]{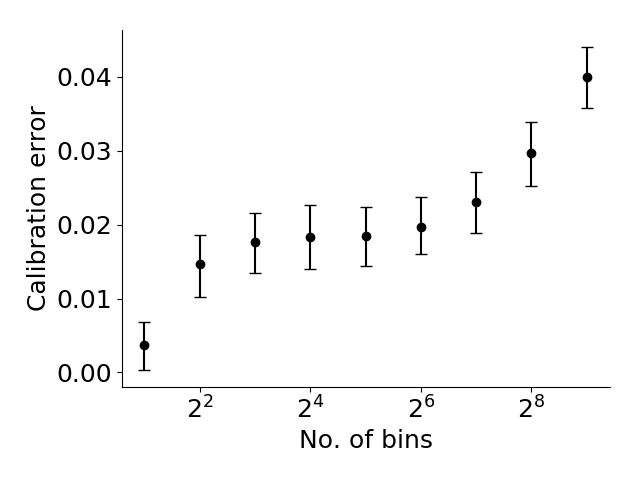}
         \caption{ImageNet}
         \label{fig:imagenet_lower_bound}
     \end{subfigure}
     \hfill
     \begin{subfigure}[b]{0.4\textwidth}
         \centering
         \includegraphics[width=\textwidth]{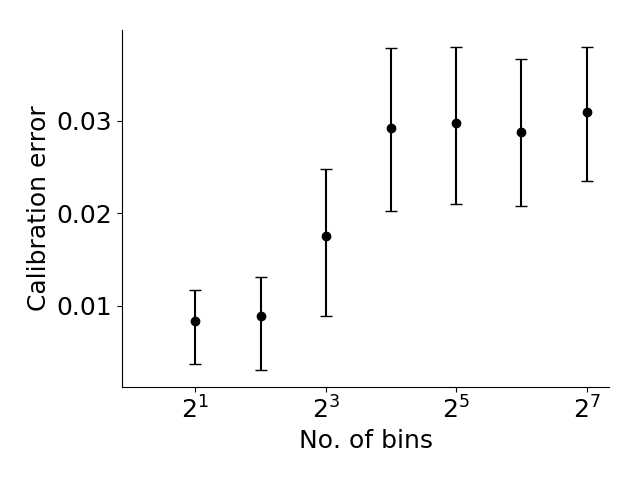}
         \caption{CIFAR-10}
         \label{fig:cifar_10_lower_bound}
     \end{subfigure}
        \caption{
          Binned calibration errors of a recalibrated VGG-net model on CIFAR-10 and ImageNet \pl{switch order; stick with the same canonical ordering of datasets} with $90\%$ confidence intervals. The binned calibration error increases as we increase the number of bins. This suggests that binning cannot be reliably used to measure the true calibration error.
        \pl{can we fix the axes to be the same across the two plots; also, maybe just write out 2, 4, 8, 16, etc. - easier to read}
        }
        \label{fig:lower_bounds}
\end{figure}






\section{\Ourcal{}}
\label{sec:calibrating_models}

Section~\ref{sec:challenges-measuring} shows that the problem with scaling methods is we cannot estimate their calibration error. The upside of scaling methods is that if the function family\pl{assume $O(1)$ parameters} has at least one function that can achieve calibration error $\epsilon$, they require $O(1/\epsilon^2)$ samples to reach calibration error $\epsilon$, while histogram binning requires $O(B/\epsilon^2)$ samples. Can we devise a method that is sample efficient to calibrate and one where it's possible to estimate the calibration error?\pl{at this point, maybe skip the rhetorical question and streamline into: we propose \ourcal{} which is both ...}
To achieve this, we propose \ourcal{} (Figure~\ref{fig:var_red_binning}) where we first fit a scaling function, and then bin the outputs of the scaling function.

\subsection{Algorithm}

We split the recalibration data $T$ of size $n$ into 3 sets: $T_1$, $T_2$, $T_3$. \Ourcal{}, illustrated in Figure~\ref{fig:variance_reduced_illustration}, outputs $\hat{g_{\mathcal{B}}}$ such that $\hat{g_{\mathcal{B}}} \circ f$ has low calibration error:

\textbf{Step 1 (Function fitting):} Select $g = \argmin_{g \in \mathcal{G}} \sum_{(z, y) \in T_1} (y - g(z))^2$.
\pl{$g$ is used in both inside argmin and out; should be $\hat g$ outside?}
\pl{say this is Platt scaling}

\textbf{Step 2 (Binning scheme construction):} We choose the bins so that an equal number of $g(z_i)$ in $T_2$\pl{make this part more precise} land in each bin $I_j$ for each $j \in \{1, \dots, B\}$---this uniform-mass binning scheme~\cite{zadrozny2001calibrated} as opposed to equal-width binning~\cite{guo2017calibration} is essential for being able to estimate the calibration error in Section~\ref{sec:verifying_calibration}.

\textbf{Step 3 (Discretization):} Discretize $g$, by outputting the average $g$ value in each bin---these are the gray circles in Figure~\ref{fig:var_red_binning}. Let $\mu(S) = \frac{1}{|S|} \sum_{s \in S} s$ denote the mean of a set of values $S$.
Let $\hat{\mu}[j] = \mu(\{ g(z_i) \; | \; g(z_i) \in I_j \wedge (z_i, y_i) \in T_3 \})$ be the mean of the $g(z_i)$ values that landed in the $j$-th bin.
Recall that if $z \in I_j$, \pl{use $\mathcal B$ as discussed before; this notation should appear in step 2 to connect things} $\beta(z) = j$ is the interval z lands in.
Then we set $\hat{g_{\mathcal{B}}}(z) = \hat{\mu}[\beta(g(z))]$---that is we simply output the mean value in the bin that $g(z)$ falls in.

\pl{why is $g$ not hatted and $\hat{g_\mathcal{B}}$ hatted? make consistent}

\ak{one reason is that I use $g_{\mathcal{B}}$ below. This is the binned version of $g$, assuming infinite data for binning. $\hat{g_\mathcal{B}}$ is the empirically binned version of $g$. The hat refers to the empirical binning. What do you think?}

\pl{sure, but then should be $\hat g$ for all of the above here}

\subsection{Analysis}

We now show that \ourcal{} requires $O(B + 1/\epsilon^2)$ samples to calibrate, and in Section~\ref{sec:verifying_calibration} we show that we can efficiently measure its calibration error. For the main theorem, we make some standard regularity assumptions on $\mathcal{G}$ which we formalize in Appendix~\ref{sec:calibrating_models_appendix}. Our result is a generalization result---we show that if $\mathcal{G}$ contains some $g^*$ with low calibration error, then our method is \emph{at least} almost as well-calibrated as $g^*$ given sufficiently many samples.

\newcommand{\finalCalibText}{
Assume regularity conditions on $\mathcal{G}$ (finite parameters, injectivity, Lipschitz-continuity, consistency, twice differentiability). Given $\delta \in (0, 1)$, there is a constant $c$ such that \emph{for all} $B, \epsilon > 0$, with $n \geq c\Big(B\log{B} + \frac{\log{B}}{\epsilon^2}\Big)$ samples, \ourcal{} finds $\hat{g}_{\mathcal{B}}$ with $(\ce(\hat{g}_{\mathcal{B}}))^2 \leq 2\min_{g \in \mathcal{G}}(\ce(g))^2 + \epsilon^2$, with probability $\geq 1 - \delta$.
}

\begin{theorem}[Calibration bound]
\label{thm:final-calib}
\finalCalibText{}
\end{theorem}
\pl{hope the $L_2$ notation can just be removed}

\newtheorem*{finalCalib}{Restatement of Theorem~\ref{thm:final-calib}}


Note that our method can potentially be better calibrated than $g^*$, because we bin the outputs of the scaling function, which reduces its calibration error (Proposition~\ref{prop:bin_low_bound}). While binning worsens the sharpness and can increase the mean-squared error of the model, in Proposition~\ref{prop:mse-finite-binning} we show that if we use many bins, binning the outputs cannot increase the mean-squared error by much.

We prove Theorem~\ref{thm:final-calib} in Appendix~\ref{sec:calibrating_models_appendix} but give a sketch here. Step 1 of our algorithm is Platt scaling, which simply fits a function $g$ to the data---standard results in asymptotic statistics show that $g$ converges in $O(\frac{1}{\epsilon^2})$ samples.
\pl{make it clear whether uniform-mass bins are needed for this theorem}

Step 3, where we bin the outputs of $g$, is the main step of the algorithm. If we had infinite data, Proposition~\ref{prop:bin_low_bound} showed that the binned version $g_{\bins{}}$ has lower calibration error than $g$, so we would be done. However we do not have infinite data---the core of our proof is to show that the empirically binned $\hat{g_{\bins{}}}$ is within $\epsilon$ of $g_{\bins{}}$ in $O(B + \frac{1}{\epsilon^2})$ samples, instead of the $O(B + \frac{B}{\epsilon^2})$ samples required by histogram binning. The intuition is in Figure~\ref{fig:variance_reduced_illustration}---the $g(z_i)$ values in each bin (gray circles in Figure~\ref{fig:var_red_binning}) are in a narrower range than the $y_i$ values (black crosses in Figure~\ref{fig:hist_binning}) and thus have lower variance so when we take the average we incur less estimation error. The perhaps surprising part is that we are estimating $B$ numbers with $\widetilde{O}(1/\epsilon^2)$ samples. In fact, there may be a small number of bins where the $g(z_i)$ values are not in a narrow range, but our proof still shows that the overall estimation error is small.

Our uniform-mass binning scheme allows us to estimate the calibration error efficiently (see Section~\ref{sec:verifying_calibration}), unlike for scaling methods where we cannot estimate the calibration error (Section~\ref{sec:challenges-measuring}).
Recall that we chose our bins so that each bin has an equal proportion of points in the recalibration set.
Lemma~\ref{lem:well-balanced} shows that this property approximately holds in the population as well.
This allows us to estimate the calibration error efficiently (Theorem~\ref{thm:final-ours}).

\begin{definition}[Well-balanced binning]
Given a binning scheme $\mathcal{B}$ of size $B$, and $\alpha \geq 1$. We say $\mathcal{B}$ is $\alpha$-well-balanced if for all $j$,
  \[ \frac{1}{\alpha B} \leq \prob(Z \in I_j) \leq \frac{\alpha}{B}\]
\end{definition}

\newcommand{\wellBalancedText}{
  For universal constant $c$, if $n \geq cB\log{\frac{B}{\delta}}$, with probability at least $1 - \delta$, the binning scheme $\bins{}$ we chose is 2-well-balanced.
}

\begin{lemma}
\label{lem:well-balanced}
\wellBalancedText{}
\end{lemma}

\newtheorem*{wellBalanced}{Restatement of Lemma~\ref{lem:well-balanced}}

While the way we choose bins is not novel~\cite{zadrozny2001calibrated}, we believe the guarantees around it are---not all binning schemes in the literature allow us to efficiently estimate the calibration error; for example, the binning scheme in~\cite{guo2017calibration} does not. Our proof of Lemma~\ref{lem:well-balanced} is in Appendix~\ref{sec:calibrating_models_appendix}. We use a discretization argument to prove the result---this gives a tighter bound than applying Chernoff bounds or a standard VC dimension argument which would tell us we need $O(B^2\log{\frac{B}{\delta}})$ samples. 

\subsection{Experiments}

Our experiments on CIFAR-10 and ImageNet show that in the low-data regime, for example when we use $\leq 1000$ data points to recalibrate, \ourcal{} produces models with much lower calibration error than histogram binning. The uncalibrated model outputs a confidence score associated with each class. We recalibrated each class separately as in~\cite{zadrozny2002transforming}, using $B$ bins per class, and evaluated calibration using the marginal calibration error (Definition~\ref{dfn:marginal-ce}).

We describe our experimental protocol for CIFAR-10.
The CIFAR-10 validation set has 10,000 data points. We sampled, with replacement, a recalibration set of 1,000 points. We ran either \ourcal{} (we fit a sigmoid in the function fitting step) or histogram binning and measured the marginal calibration error on the entire set of 10K points.
We repeated this entire procedure 100 times and computed mean and 90\% confidence intervals, and we repeated this varying the number of bins $B$. Figure~\ref{fig:marginal_calibrator_comparison_cifar} shows that \ourcal{} produces models with lower calibration error, for example $35\%$ lower calibration error when we use 100 bins per class.
\pl{what method do we use to estimate calibration error, plugin? should say this and say that we could use a fancier method (Section 5)}

Using more bins allows a model to produce more fine-grained predictions, e.g.~\cite{brocker2012empirical} use $B = 51$ bins, which improves the quality of predictions as measured by the mean-squared error---Figure~\ref{fig:cifar_calibrator_cmp_mse_ce} shows that our method achieves better mean-squared errors for any given calibration constraint.
More concretely, the figure shows a scatter plot of the mean-squared error and squared calibration error for histogram binning and \ourcal{} when we vary the number of bins. For example, if we want our models to have a calibration error $\leq 0.02 = 2\%$ we get a $9\%$ lower mean-squared error. In Appendix~\ref{sec:calibrating_models_appendix_experiments} we show that we get \emph{5x lower top-label calibration error on ImageNet}, and give further experiment details.
\pl{bit confused - this paragraph seems to be mostly about using more bins to lower MSE, but then we're talking about lower calibration error again;
based on the Fig 3b, what I gathered is that simply because calibration error is lower for scaling-binning, for the sample calibration error,
you can use more bins, which corresponds to a lower MSE
}

\textbf{Validating theoretical bounds}: In Appendix~\ref{sec:calibrating_models_appendix_experiments} we run synthetic experiments to validate the bound in Theorem~\ref{thm:final-calib}. In particular, we show that if we fix the number of samples $n$, and vary the number of bins $B$, the squared calibration error for \ourcal{} is nearly constant but for histogram binning increases nearly linearly with $B$. For both methods, the squared calibration error decreases approximately as $1/n$---that is when we double the number of samples the squared calibration error halves.


\begin{figure}
  \centering
  \centering
     \begin{subfigure}[b]{0.54\textwidth}
         \centering
         \includegraphics[width=0.8\textwidth]{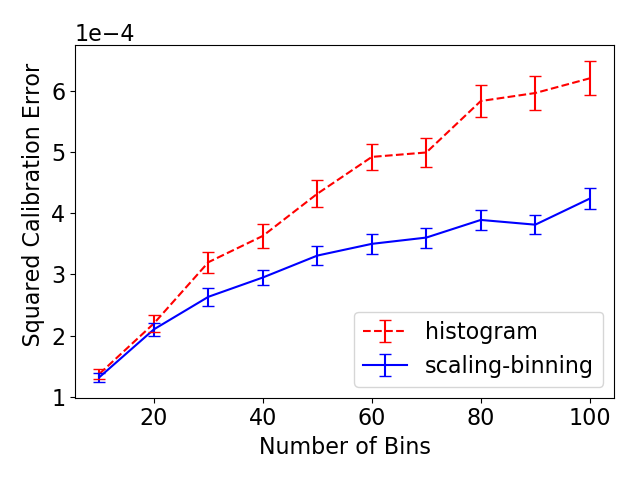}
         \caption{Effect of number of bins on squared calibration error. \pl{make vertical axis start at 0}}
         \label{fig:marginal_calibrator_comparison_cifar}
     \end{subfigure}
     \hfill
     \begin{subfigure}[b]{0.44\textwidth}
         \centering
         \includegraphics[width=\textwidth]{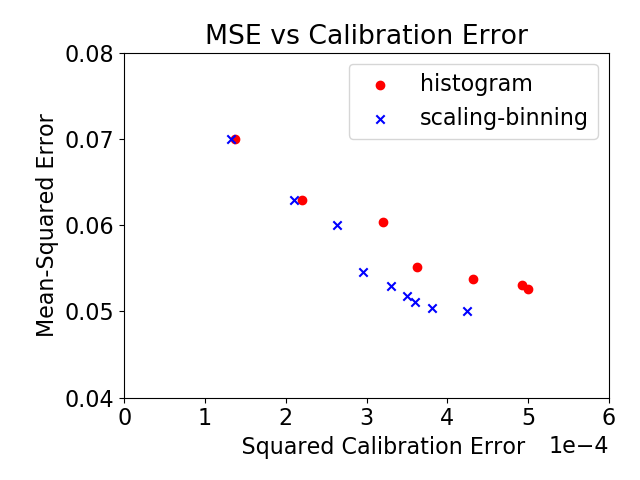}
         \caption{Tradeoff between calibration and MSE. \pl{get rid of title of plot since we already have caption}}
         \label{fig:cifar_calibrator_cmp_mse_ce}
     \end{subfigure}
  \caption{
  (\textbf{Left}) Recalibrating using 1,000 data points on CIFAR-10, \ourcal{} achieves lower squared calibration error than histogram binning, especially when the number of bins $B$ is large.
  (\textbf{Right}) For a fixed calibration error, \ourcal{} allows us to use more bins. This results in models with more predictive power which can be measured by the mean-squared error. Note the vertical axis range is $[0.04, 0.08]$ to zoom into the relevant region.
  \pl{axis labels say L2 squared... in text we say $L^2$ - this is confusing; make consistent}
  }
  \label{fig:nan2}
\end{figure}

\section{Verifying calibration}
\label{sec:verifying_calibration}

\newcommand{\sqrterror}[0]{\ensuremath{\mathcal{E}}}
\newcommand{\error}[0]{\ensuremath{\mathcal{E}^2}}
\newcommand{\errorEst}[0]{\ensuremath{\hat{\mathcal{E}}^2}}
\newcommand{\pluginEst}[0]{\ensuremath{\hat{\mathcal{E}}_{\textup{pl}}^2}}
\newcommand{\debiasedEst}[0]{\ensuremath{\hat{\mathcal{E}}_{\textup{db}}^2}}
\newcommand{\piSmallBound}[0]{\ensuremath{\frac{12}{n}\log{\frac{2B}{\delta}}}}

\tm{I wonder we should have subscript for $\debiasedEst$. $\ensuremath{\hat{E}_{\textup{cxl}}^2}$? Sounds a bit verbose..} \ak{Ah are you suggesting we should, or should not, have these subscripts? I removed it for the debiased estimator, but kept it for plugin.}
Before deploying our model we would like to check that it has calibration error below some desired threshold $\sqrterror{}$\pl{is this actually how we'd use our technique? I thought the measuring calibration error is just an offline thing we'd to do validate the calibration method; otherwise that's a lot of samples we need}. In this section we show that we can accurately estimate the calibration error of binned models, if the binning scheme is 2-well-balanced. Recent work in machine learning uses a plugin estimate for each term in the calibration error~\cite{nguyen2015posterior, hendrycks2019anomaly, kuleshov2015calibrated, hendrycks2019pretraining}. Older work in meteorology~\cite{brocker2012empirical, ferro2012bias} notices that this is a biased estimate, and proposes a \emph{debiased} estimator that subtracts off an approximate correction term to reduce the bias. Our contribution is to show that the debiased estimator is more accurate: while the plugin estimator requires samples proportional to $B$ to estimate the calibration error, the debiased estimator requires samples proportional to $\sqrt{B}$. Note that we show an \emph{improved sample complexity}---prior work only showed that the naive estimator is biased. In Appendix~\ref{sec:verifying_calibration_appendix_experiments} we also propose a way to debias the $\ell_1$ calibration error (ECE), and show that we can estimate the ECE more accurately on CIFAR-10 and ImageNet.

Suppose we wish to measure the squared calibration error $\error{}$ of a binned model $f : \mathcal{X} \to S$ where $S \subseteq [0, 1]$ and $|S| = B$. Suppose we get an evaluation set $T_n = \{(x_1, y_1), \dots, (x_n, y_n)\}$.\pl{notation clashes with the $T_1,T_2,T_3$ from before, which I think we should change} Past work typically estimates the calibration error by directly\pl{again, not meaningful} estimating each term from samples:

\begin{restatable}[Plugin estimator]{definition}{pluginDfn}
\label{dfn:plugin-estimator}
  Let $L_s$ denote the $y_j$ values where the model outputs $s$: $L_s = \{ y_j \; | \; (x_j, y_j) \in T_n\wedge f(x_j) = s \}$. Let $\hat{p}_s$ be the estimated probability of $f$ outputting $s$:
$\hat{p}_s = \frac{|L_s|}{n}$.

Let $\hat y_s$ be the empirical average of $Y$ when the model outputs $s$: $\hat y_s = \sum_{y \in L_s} \frac{y}{|L_s|}$.

  The plugin estimate for the squared calibration error is the weighted squared difference between $\hat y_s$ and $s$:
\[ \pluginEst{} = \sum_{s \in S} \hat{p}_s (s - \hat y_s)^2 \]
  \pl{number all equations}
\end{restatable}
\pl{I feel like we should be sharing notation...when we talk about binning, it's the same stuff}

Alternatively, \cite{brocker2012empirical, ferro2012bias} propose to subtract an approximation of the bias from the estimate:

\begin{restatable}[Debiased estimator]{definition}{cancelingDfn}
  The debiased estimator for the squared calibration error is:
\[ \debiasedEst{} = \sum_{s \in S} \hat{p}_s \Big[ (s - \hat{y}_s)^2 - \frac{\hat{y}_s(1 - \hat{y}_s)}{\hat{p}_sn-1} \Big] \]
  \pl{number equations}
  \pl{give a name (put it behind a macro) $\hat E_\text{db}$}
\end{restatable}

We are interested in analyzing the number of samples required to estimate the calibration error within a constant multiplicative factor, that is to give an estimate $\errorEst{}$ such that $\lvert \errorEst{} - \error{} \rvert \leq \frac{1}{2}\error{}$ (where $\frac{1}{2}$ can be replaced by any constant $r$ with $0 < r < 1$). Our main result is that the plugin estimator requires $\widetilde{O}(\frac{B}{\error{}})$ samples (Theorem~\ref{thm:final-plugin}) while the debiased estimator requires $\widetilde{O}(\frac{\sqrt{B}}{\error{}})$ samples (Theorem~\ref{thm:final-ours}).

\newcommand{\finalPluginText}{
Suppose we have a binned model with squared calibration error $\error{}$, where the binning scheme is 2-well-balanced, that is for all $s \in S$, $\prob(f(X) = s) \geq \frac{1}{2B}$.
  \footnote{We do not need the upper bound of the 2-well-balanced property. \pl{can we just define the property with the one-sided bound; I don't think the upper bound is every useful}}
  If $n \geq c\frac{B}{\error{}}\log{\frac{B}{\delta}}$ for some universal constant $c$, then for the plugin estimator, we have $\frac{1}{2} \error{} \leq \pluginEst{} \leq \frac{3}{2} \error{}$ with probability at least $1 - \delta$.
}

\begin{theorem}[Plugin estimator bound]
\label{thm:final-plugin}
\finalPluginText{}
\end{theorem}

\newtheorem*{finalPlugin}{Restatement of Theorem~\ref{thm:final-plugin}}

\newcommand{\finalCancelingText}{
Suppose we have a binned model with squared calibration error $\error{}$ and for all $s \in S$, $\prob(f(X) = s) \geq \frac{1}{2B}$. If $n \geq c\frac{\sqrt{B}}{\error{}}\log{\frac{B}{\delta}}$ for some universal constant $c$ then for the debiased estimator, we have $\frac{1}{2} \error{} \leq \debiasedEst{} \leq \frac{3}{2}\error{}$ with probability at least $1 - \delta$.
}

\begin{theorem}[Debiased estimator bound]
\label{thm:final-ours}
\finalCancelingText{}
\end{theorem}

\newtheorem*{finalCanceling}{Restatement of Theorem~\ref{thm:final-ours}}


The proofs of both theorems is in Appendix~\ref{sec:verifying_calibration_appendix}. The idea is that for the plugin estimator, each term in the sum has bias $1/n$. These biases accumulate, giving total bias $B/n$. The debiased estimator has much lower bias and the estimation variance cancels across bins---this intuition is captured in Lemma~\ref{lem:c3_bound} which requires careful conditioning to make the argument go through.


\newcommand{\calset}[0]{\ensuremath{S_C}}
\newcommand{\verifset}[0]{\ensuremath{S_E}}

\subsection{Experiments}
\label{sec:verifying_calibration_experiments}

\begin{figure}
  \centering
  \centering
     \begin{subfigure}[b]{0.45\textwidth}
         \centering
         \includegraphics[width=\textwidth]{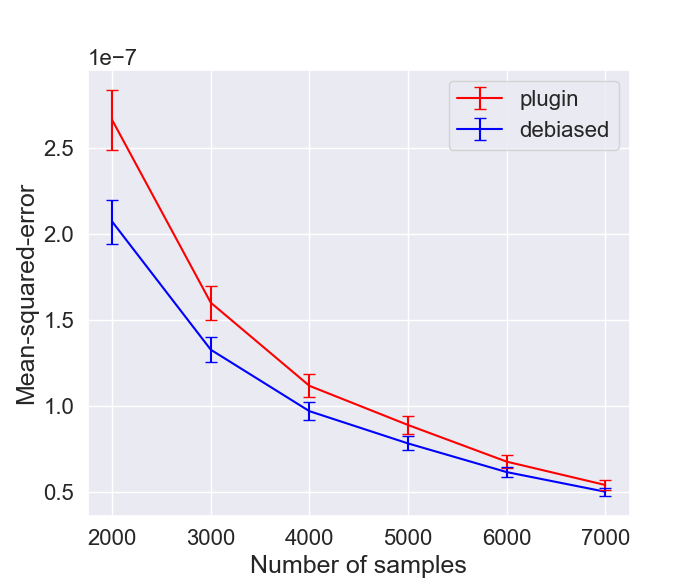}
         \caption{$B = 10$
         }
         \label{fig:mse_estimators}
     \end{subfigure}
     \hfill
     \begin{subfigure}[b]{0.45\textwidth}
         \centering
         \includegraphics[width=\textwidth]{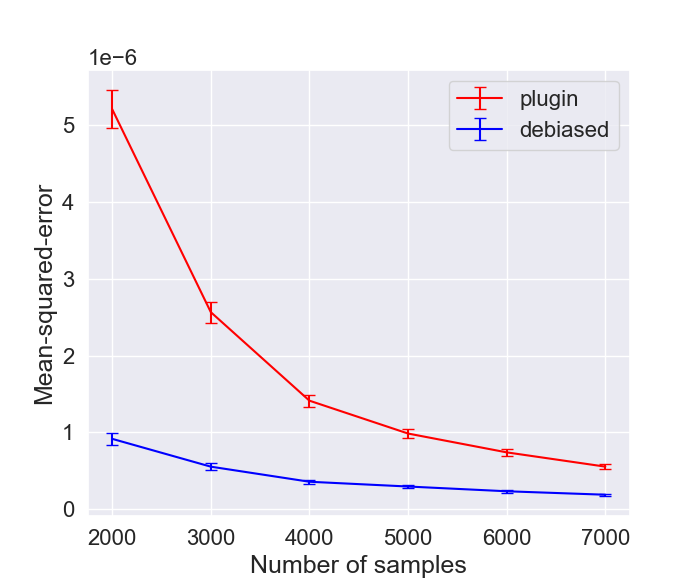}
         \caption{$B = 100$
         }
         \label{fig:ce_vs_bins_verifying}
     \end{subfigure}
  \caption{
    Mean-squared errors of plugin and debiased estimators on a recalibrated VGG16 model on CIFAR-10 with $90\%$ confidence intervals (lower values better). The debiased estimator is closer to the ground truth, which corresponds to $0$ on the vertical axis, especially when $B$ is large or $n$ is small.
    Note that this is the MSE of the squared calibration error, not the MSE of the model in Figure~\ref{fig:nan2}.
    \pl{important: saying MSE is confusing here because it's the MSE of the calibration error (so meta), not to be confused with Fig 3 (b), which is the MSE of the model;
    should say that in both the text and vertical axis}
}
  \label{fig:mse_estimators_bins}
\end{figure}

We run a multiclass marginal calibration experiment on CIFAR-10 which suggests that the debiased estimator produces better estimates of the calibration error than the plugin estimator. We split the validation set of size 10,000 into two sets $\calset{}$ and $\verifset{}$ of sizes 3,000 and 7,000 respectively. We use $\calset{}$ to re-calibrate and discretize a trained VGG-16 model. We calibrate each of the $K = 10$ classes seprately as described in Section~\ref{sec:formulation} and used $B = 100$ or $B = 10$ bins per class. For varying values of $n$, we sample $n$ points with replacement from $\verifset{}$, and estimate the calibration error using the debiased estimator and the plugin estimator. We then compute the squared deviation of these estimates from the squared calibration error measured on the entire set $\verifset{}$. We repeat this resampling 1,000 times to get the mean squared deviation of the estimates from the ground truth and confidence intervals. Figure~\ref{fig:mse_estimators} shows that the debiased estimates are much closer to the ground truth than the plugin estimates---the difference is especially significant when the number of samples $n$ is small or the number of bins $B$ is large. Note that having a perfect estimate corresponds to $0$ on the vertical axis.

In Appendix~\ref{sec:verifying_calibration_appendix_experiments}, we include histograms of the absolute difference between the estimates and ground truth for the plugin and debiased estimator, over the 1,000 resamples.

\section{Related work}

Calibration, including the squared calibration error, has been studied in many fields besides machine learning including meteorology~\cite{murphy1973vector, murphy1977reliability, degroot1983forecasters,gneiting2005weather, brocker2009decomposition}, fairness~\cite{johnson2018multicalibration, liu2019implicit}, healthcare~\cite{jiang2012calibrating, crowson2017calibration, harrell1996prognostic, yadlowsky2019calibration}, reinforcement learning~\cite{malik2019calibrated}, natural language processing~\cite{nguyen2015posterior, card2018calibration}, speech recognition~\cite{dong2011calibration}, econometrics~\cite{gneiting2007probabilistic}, and psychology~\cite{lichtenstein1982calibration}.
Besides the calibration error, prior work also uses the Hosmer-Lemeshov test~\cite{hosmer1980goodness} and reliability diagrams~\cite{degroot1983forecasters, brocker2007reliability} to evaluate calibration.
Concurrent work to ours~\cite{widmann2019calibration} also notice that using the plugin calibration error estimator to test for calibration leads to rejecting well-calibrated models too often.
Besides calibration, other ways of producing and quantifying uncertainties include Bayesian methods~\cite{gelman1995bayesian} and conformal prediction~\cite{shafer2008tutorial, lei2016distribution}.

Recalibration is related to (conditional) density estimation~\cite{wasserman2019, parzen1962} as the goal is to estimate $\expect[Y \mid f(X)]$. Algorithms and analysis in density estimation typically assume the true density is $L-$Lipschitz, while in calibration applications, the calibration error of the final model should be measurable from data, without making untestable assumptions on $L$.

Bias is a common issue with statistical estimators, for example, the seminal work by Stein~\cite{stein81sure} fixes the bias of the mean-squared error. However, debiasing an estimator does not typically lead to \emph{an improved sample complexity}, as it does in our case.

\section{Conclusion}

This paper makes three contributions: 1. We showed that the calibration error of continuous methods is underestimated; 2. We introduced the first method, to our knowledge, that has better sample complexity than histogram binning and has a \emph{measurable calibration error}, giving us the best of scaling and binning methods; and 3. We showed that an alternative estimator for calibration error has better sample complexity than the plugin estimator.
There are many exciting avenues for future work:
\begin{enumerate}
\item \textbf{Dataset shifts}: Can we maintain calibration under dataset shifts? When the dataset shifts (for example, train on MNIST, but evaluate on SVHN) it is difficult to get very high accuracies on the target dataset, but can we at least know when our model is accurate and when it is not, or in other words can we ensure our model is calibrated on the new dataset? When a small amount of labeled \emph{target} data is available, we can simply re-calibrate using the labeled target data since recalibration simply involves tuning a small number of parameters and therefore requires very few samples. However, can we maintain calibration when we do not have labeled examples from the target dataset? 
\item \textbf{Measuring calibration}: Can we come up with alternative metrics that still capture a notion of calibration, but are measurable for scaling methods?
\item \textbf{Multiclass calibration}: Most papers on calibration focus on top-label calibration---can we come up with better methods for marginal calibration, so that each class prediction is calibrated? Can we achieve stronger notions of calibration, such as joint calibration, efficiently? 
\item \textbf{Better binning}: Can we devise better ways of binning? While binning makes the calibration error measurable, it leads to an increase in the mean-squared error. In Proposition~\ref{prop:mse-finite-binning} we bound the increase in MSE for the well-balanced binning scheme. However, in Section~\ref{sec:alt_binning_schemes} we give intuition for why other binning schemes may have a smaller increase in MSE.
\item \textbf{Finite sample guarantees}: Can we prove finite sample guarantees for Theorem~\ref{thm:final-calib} which show the dependency on the dimension and probability of failure?
\item \textbf{Estimating calibration}: The plugin estimator for the $\ell_1$ calibration error (called the ECE in~\cite{guo2017calibration, nixon2019calibration}) is also biased---can we come up with a more sample efficient estimator for the $\ell_1$ calibration error? In Appendix~\ref{sec:verifying_calibration_appendix_experiments} we also propose a heuristic way to debias the $\ell_1$ calibration error and show strong experimental results, but can we prove guarantees for this? For the $\lsquared$ calibration error, can we devise an even better estimator? If not, can we prove minimax lower bounds on this?
\end{enumerate}

\paragraph{Reproducibility.}
Our Python calibration library is available at \url{https://pypi.org/project/uncertainty-calibration}.
All code, data, and experiments can be found on CodaLab at \url{https://worksheets.codalab.org/worksheets/0xb6d027ee127e422989ab9115726c5411}.
Updated code can be found at \url{https://github.com/AnanyaKumar/verified_calibration}.

\paragraph{Acknowledgements.}

The authors would like to thank the Open Philantropy Project, Stanford Graduate Fellowship, and Toyota Research Institute for funding. Toyota Research Institute (``TRI") provided funds to assist the authors with their research but this article solely reflects the opinions and conclusions of its authors and not TRI or any other Toyota entity.

We are grateful to Pang Wei Koh, Chuan Guo, Anand Avati, Shengjia Zhao, Weihua Hu, Yu Bai, John Duchi, Dan Hendrycks, Jonathan Uesato, Michael Xie, Albert Gu, Aditi Raghunathan, Fereshte Khani, Stefano Ermon, Eric Nalisnick, and Pushmeet Kohli for insightful discussions. We thank the anonymous reviewers for their thorough reviews and suggestions that have improved our paper. We would also like to thank Pang Wei Koh, Yair Carmon, Albert Gu, Rachel Holladay, and Michael Xie for their inputs on our draft, and Chuan Guo for providing code snippets from their temperature scaling paper.






\bibliographystyle{unsrtnat}
\bibliography{local,refdb/all}

\appendix

\newpage 
\section{Model and code details}

The VGG16 model we used for ImageNet experiments is from the Keras~\cite{chollet2015} module in the TensorFlow~\cite{tensorflow2015-whitepaper} library and we used pre-trained weights supplied by the library. The VGG16 model for CIFAR-10 was obtained from an open-source implementation on GitHub~\cite{geifman2017}, and we used the pre-trained weights there. We independently verified the accuracies of these models.

\newpage
\section{Proofs for Section~\ref{sec:challenges-measuring}}
\label{sec:appendix-platt-not-calibrated}



The results in Section~\ref{sec:challenges-measuring} hold more generally for $\lpce{}$ and not just $\ce{}$. We recall the definition of $\lpce{}$:
\begin{definition}
For $p \geq 1$, the $\ell_p$ calibration error of $f : \mathcal{X} \to [0, 1]$ is given by:
\begin{align}
\lpce(f) = \Big(\expect\big[ \left|f(X) - \expect[Y \mid f(X)] \right|^p \big] \Big)^{1/p}
\end{align}
\end{definition}

We now restate and prove the results in terms of $\lpce{}$ which implies the result for $\ce{}$ in the paper as a special case, but also for other commonly used metrics such as the $\ell_1$ calibration error (ECE).

\begin{continuousNotCalibrated}
For any binning scheme $\bins{}$, $p \geq 1$, and continuous bijective function $f : [0, 1] \to [0, 1]$, there exists a distribution \pl{mathbb}$P$ over $\mathcal{X}, \mathcal{Y}$ s.t. $\lpce(f_{\bins{}}) = 0$ but $\lpce(f) \geq 0.49$.
Note that for all $f$, $0 \leq \lpce(f) \leq 1$.
\end{continuousNotCalibrated}

\begin{proof}
As stated in the main text, the intuition is that in each interval $I_j$ in $\bins{}$, the model could underestimate the true probability $\expect[Y \mid f(X)]$ half the time, and overestimate the probability half the time. So if we average over the entire bin the model appears to be calibrated, even though it is very uncalibrated. The proof simply formalizes this intuition.

Since $f$ is bijective and continuous we can select data distribution $P$ s.t. $f(X) \sim \mbox{Uniform}[0.5 - \epsilon, 0.5 + \epsilon]$ for any $\epsilon > 0$. To see this, first note that from real analysis since $f : [0, 1] \to [0, 1]$ and $f$ is bijective and continuous, $f^{-1}$ is also bijective and continuous.
Then we can let $X \sim f^{-1}(\mbox{Uniform}[0.5 - \epsilon, 0.5 + \epsilon])$ which has the desired property and has a density.

Now, consider each interval $I_j$ in binning scheme $\bins{}$, and let $A_j = I_j \cap \mbox{Uniform}[0.5 - \epsilon, 0.5 + \epsilon]$.
If $A_j = \emptyset$ then $P(f(X) \in A_j) = 0$ so we can ignore this interval (since $f(X)$ will never land in this bin).
Let $p_j = \expect[f(X) \mid f(X) \in A_j]$.
Note that $\expect[f(X) \mid f(X) \in A_j] = \expect[f(X) \mid f(X) \in I_j]$.
Since $f(X) \in [0.5 - \epsilon, 0.5 + \epsilon]$, $p_j \in [0.5 - \epsilon, 0.5 + \epsilon]$ as well.
We will choose $P(Y)$ so that $Y$ is $1$ whenever $f(X)$ lands in the first $p_j$ fraction of interval $A_j$, and $0$ whenever $f(X)$ lands in the latter $1 - p_j$ fraction of $A_j$.
Then $\expect[Y \mid f(X) \in A_j] = p_j$, so the binned calibration error is 0.
But notice that for all $s \in [0.5 - \epsilon, 0.5 + \epsilon]$, $\expect[Y \mid f(X) = s]$ is either $0$ or $1$.
So we have:
\[ \lvert \expect[Y \mid f(X) = s] - s \rvert \geq 0.5 - \epsilon \]
That is, at every point the model is actually very miscalibrated at each $s$. By taking $\epsilon$ very small, we then get that $\lpce(p) \geq 0.5 - \epsilon'$ for any $\epsilon' > 0$, which completes the proof.
\end{proof}

\begin{binningLowerBound}
Given any binning scheme $\bins{}$ and model $f : \mathcal{X} \to [0, 1]$, we have:
\[  \lpce(f_{\bins{}}) \leq \lpce(f). \]
\end{binningLowerBound}

\begin{proof}
It suffices to prove the claim for the $\ell_p^p$ error:
\[ (\lpce(f_{\bins{}}))^p \leq (\lpce(f))^p \]
This is because if $p > 0$ then $a \leq b \Leftrightarrow a^p \leq b^p$.

For $p \geq 1$, let $l(a, b) = (|a - b|)^p$.
We note that $l$ is convex in both arguments.
The proof is now a simple result of Jensen's inequality and convexity of $l$.
Suppose that $\mathcal{B}$ is given by intervals $I_1, ..., I_B$.
Let $Z = f(X)$---note that $Z$ is a random variable.

We can write $(\lpce(f_{\bins{}}))^p$ as:
\[ (\lpce(f_{\bins{}}))^p = \sum_{j=1}^B P(Z \in I_j) \; l\Big( \mathbb{E}[Z \mid Z \in I_j], \mathbb{E}[Y \mid Z \in I_j] \Big) \]
We can write $(\lpce(f))^p$ as:
\[ (\lpce(f))^p = \sum_{j=1}^B P(Z \in I_j) \; \mathbb{E}\Big[ l\big( Z, \mathbb{E}[Y \mid Z] \big) \mid Z \in I_j \Big] \]
Fix some bin $I_j \in \bins{}$. By Jensen's inequality,
\[ l\Big( \mathbb{E}[Z \mid Z \in I_j], \mathbb{E}[Y \mid Z \in I_j] \Big) \leq \mathbb{E}\Big[ l\big( Z, \mathbb{E}[Y \mid Z] \big) \mid Z \in I_j \Big] \]
Since this inequality holds for each term in the sum, it holds for the whole sum:
\[ (\lpce(f_{\bins{}}))^p \leq (\lpce(f))^p \]
Note that the proof also implies that finer binning schemes give a better lower bound.
That is, given $\bins{}'$ suppose for all $I_j' \in \bins{}'$, $I_j' \subseteq I_k$ for some $I_k \in \bins{}$.
Then $\lpce(f_{\bins{}}) \leq \lpce(f_{\bins{}'}) \leq \lpce(f)$.
This is because $f_{\bins{}'}$ can be seen as a binned version of $f_{\bins{}}$.


\end{proof}

\newpage
\section{Ablations for Section~\ref{sec:challenges-measuring}}
\label{sec:appendix_platt_experiments}

Here we present additional experiments for Section~\ref{sec:challenges-measuring}.
Recall that the experiments in section~\ref{sec:challenges-measuring} showed that binning underestimates the calibration error of a model---we focused on the $\ell_2\mbox{-CE}$ and selected bins so that each bin has an equal number of data points. Figure~\ref{fig:imagenet_lower_bound_l1} shows that binning is also unreliable at measuring the $\ell_1\mbox{-CE}$ (ECE) on ImageNet---using more bins uncovers a higher calibration error than we would otherwise detect with fewer bins. Figure~\ref{fig:imagenet_lower_bound_l1_prob} shows that the same conclusion holds on ImageNet if we look at the $\ell_1\mbox{-CE}$ \emph{and} use an alternative approach to selecting bins used in~\cite{guo2017calibration} that we call \emph{equal-width binning}. Here, the $B$ bins are selected to be $I_1 = [0, \frac{1}{B}], I_2 = (\frac{1}{B}, \frac{2}{B}], \dots, I_B = (\frac{B-1}{B}, 1]$. The experimental protocol is the same as in section~\ref{sec:challenges-measuring}.

\begin{figure}
     \centering
     \begin{subfigure}[b]{0.45\textwidth}
         \centering
         \includegraphics[width=\textwidth]{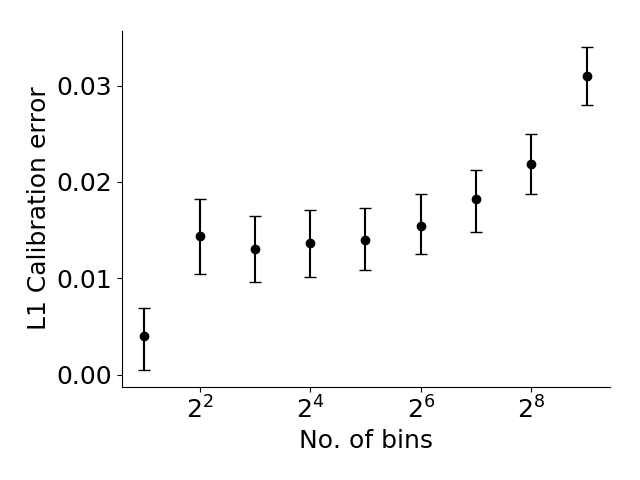}
         \caption{ImageNet, $\ell_1\mbox{-CE}$}
         \label{fig:imagenet_lower_bound_l1}
     \end{subfigure}
     \hfill
     \begin{subfigure}[b]{0.45\textwidth}
         \centering
         \includegraphics[width=\textwidth]{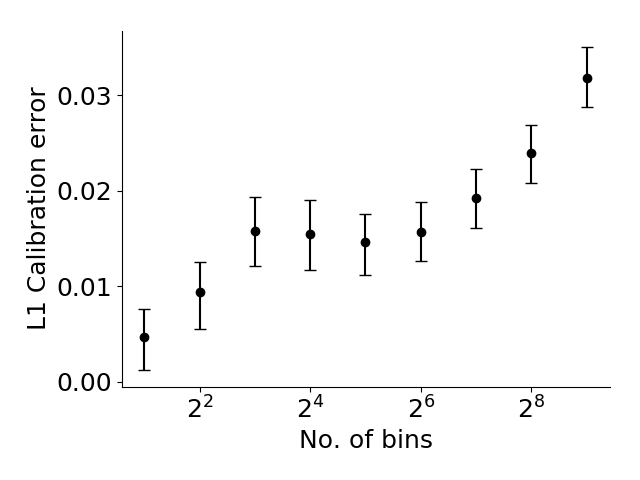}
         \caption{ImageNet, $\ell_1\mbox{-CE}$, equal-width binning}
         \label{fig:imagenet_lower_bound_l1_prob}
     \end{subfigure}
        \caption{
        Binned $\ell_1$ calibration errors of a recalibrated VGG-net model on ImageNet with $90\%$ confidence intervals. The binned calibration error increases as we increase the number of bins. This suggests that binning cannot be reliably used to measure the $\ell_1\mbox{-CE}$.
        }
        \label{fig:lower_bounds_l1_imagenet}
\end{figure}

We repeated both of these experiments on CIFAR-10 as well, and plot the results in Figure~\ref{fig:lower_bounds_l1_cifar}. Here the results are inconclusive because the error bars are large. This is because the CIFAR-10 dataset is smaller than ImageNet, and the accuracy of the CIFAR-10 model is 93.1\%, so the calibration error that we are trying to measure is much smaller.

\begin{figure}
     \centering
     \begin{subfigure}[b]{0.45\textwidth}
         \centering
         \includegraphics[width=\textwidth]{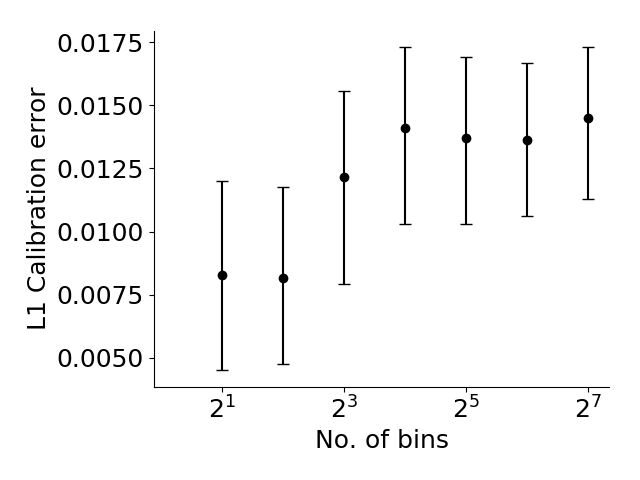}
         \caption{CIFAR-10, $\ell_1\mbox{-CE}$}
         \label{fig:cifar_lower_bound_l1}
     \end{subfigure}
     \hfill
     \begin{subfigure}[b]{0.45\textwidth}
         \centering
         \includegraphics[width=\textwidth]{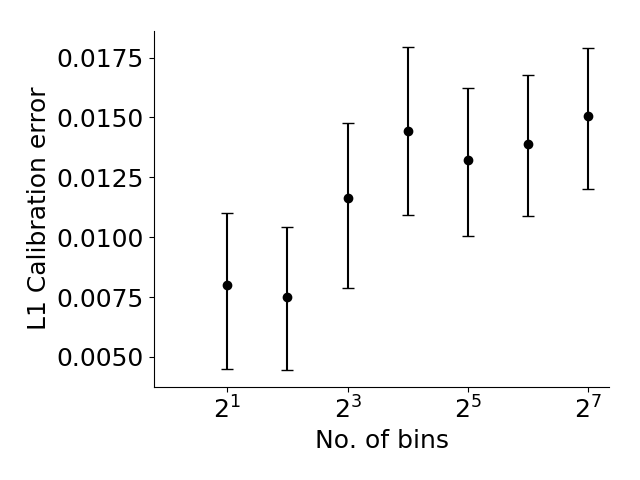}
         \caption{CIFAR-10, $\ell_1\mbox{-CE}$, equal-width binning}
         \label{fig:cifar_lower_bound_l1_prob}
     \end{subfigure}
        \caption{
        Binned $\ell_1$ calibration errors of a recalibrated VGG-net model on CIFAR-10 with $90\%$ confidence intervals. The results are not as conclusive here because the error bars are large, however it seems to suggest that the binned calibration error increases as we increase the number of bins.
        }
        \label{fig:lower_bounds_l1_cifar}
\end{figure}

We provide details on the dataset split for CIFAR-10. For CIFAR-10, we used a VGG16 model and split the test set into 3 sets of size $(1000, 1000, 8000)$, where used the first set of data to recalibrate the model using Platt scaling, the second to select the binning scheme, and the third to measure the binned calibration error. As stated in the main body of the paper, for ImageNet we used a split of $(20000, 5000, 25000)$.

\newpage
\section{Proofs for section~\ref{sec:calibrating_models}}
\label{sec:calibrating_models_appendix}

\newcommand{\G}[0]{\ensuremath{\mathcal{G}}}

Our analysis of the sample complexity of \ourcal{} requires some assumptions on the function family $\G{}$:

\begin{enumerate}
  \item (Finite\pl{-dimensional} parameters). Let $\G{} = \{ g_{\theta} : [0, 1] \to [0, 1] \; | \; \theta \in A \}$ where $A \subseteq \mathbb{R}^{d}$ and $A$ is open.
\item (Injective). For all $g_{\theta} \in \G{}$ we assume $g_{\theta}$ is injective.
\item (Consistency). Intuitively, consistency means that given infinite data, the estimated parameters should converge to the unique optimal parameters in $A$.
More formally, suppose $\theta^* = \argmin_{\theta \in A} \mse(g_{\theta})$.
Then the parameters $\hat{\theta}_n$ estimated by minimizing the empirical MSE with $n$ samples in step 1 of the algorithm, converges in distribution to $\theta^*$, that is, $\hat{\theta}_n \to_D \theta^*$ as $n \to \infty$. Note that consistency inherently assumes identifiability, that there is a unique minimizer $\theta^*$ in the open set $A$.
\item (Regularity). We assume that regularity conditions in Theorem 5.23 of~\cite{vaart98asymptotic} hold, which require the loss to be twice differentiable with symmetric, non-singular Hessian, and that $g_{\theta}(x)$ is Lipschitz in $\theta$ for all $x$. We will also require the second derivative to be continuous.
\end{enumerate}

We assume that $\G{}$ satisfies these assumptions in the rest of this section. Note that aside from injectivity, the remaining conditions are only required for the (fairly standard) analysis of step 1 of the algorithm, which says that a parametric scaling method with a small number of parameters will quickly converge to its optimal error.



\subsection{Calibration bound (Proof of Theorem~\ref{thm:final-calib})}

The goal is to prove the following theorem from Section~\ref{sec:calibrating_models}, which we restate:

\begin{finalCalib}
\finalCalibText{}
\end{finalCalib}

We will analyze each step of our algorithm and then combine the pieces to get Theorem~\ref{thm:final-calib}.
As we mention in the main text, step 3 is the main step, so Lemma~\ref{thm:empirical-binning} is one of the core parts of our proof.
Step 2 is where we construct a binning scheme so that each bin has an equal number of points---we show that this property holds approximately in the population (Lemma~\ref{lem:well-balanced}).
This is important as well, particularly to ensure we can estimate the calibration error.
Step 1 is basically Platt scaling, and the asymptotic analysis is fairly standard.

\textbf{Step 3}: Our proofs will require showing convergence in $\ell_2$ and $\ell_1$ norm in function space, which we define below:

\begin{definition}[Distances between functions]
Given $f, g : [0, 1] \to [0, 1]$, for the $\ell_2$ norm we define $||f - g||_2^2 = \expect[(f(Z) - g(Z))^2]$ and $||f- g||_2 = \sqrt{||f - g||_2^2}$. For the $\ell_1$ norm we define $||f - g||_1 = \expect[\lvert f(Z) - g(Z)\rvert]$
\end{definition}

Recall that we showed that in the limit of infinite data the binned version of $g$, $g_{\bins{}}$, has lower calibration error than $g$ (Proposition~\ref{prop:bin_low_bound}). However our method uses $n$ data points to empirically bin $g$, giving us $\hat{g_{\bins{}}}$. We now show the key lemma that allows us to bound the calibration error and later the mean-squared error. That is, we show that the empirically binned function $\hat{g_{\bins{}}}$ quickly converges to $g_{\bins{}}$ in both $\ell_2$ and $\ell_1$ norms.

\begin{lemma}[Empirical binning]
\label{thm:empirical-binning}
There exist constants $c_B, c_1, c_2$ such that the following is true. Given $g : [0, 1] \to [0, 1]$, binning set $T_3 = \{(z_i, y_i)\}_{i=1}^n$ and a 2-well-balanced binning scheme $\bins{}$ of size $B$. Given $0 < \delta < 0.5$, suppose that $n \geq c_B B\log{\frac{B}{\delta}}$. Then with probability at least $1 - \delta$,  $||\hat{g_{\bins{}}} - g_{\bins{}}||_2 \leq \frac{c_2}{\sqrt{n}}\sqrt{\log{\frac{B}{\delta}}}$ and $||\hat{g_{\bins{}}} - g_{\bins{}}||_1 \leq \frac{c_1}{\sqrt{nB}}\sqrt{\log{\frac{B}{\delta}}}$
\end{lemma}

\begin{proof}

Recall that the intuition is in Figure~\ref{fig:variance_reduced_illustration} of the main text---the $g(z_i)$ values in each bin (gray circles in Figure~\ref{fig:var_red_binning}) are in a narrower range than the $y_i$ values (black crosses in Figure~\ref{fig:hist_binning}) so when we take the average we incur less of an estimation error. Now, there may be a small number of bins where the $g(z_i)$ values are not in a narrow range, but we will use the assumption that $\bins{}$ is 2-well-balanced to show that these effects average out and the overall estimation error is small.

Define $R_j$ to be the set of $g(z_i)$ that fall into the $j$-th bin, given by $R_j = \{g(z_i) \mid g(z_i) \in I_j \wedge (z_i, y_i) \in T_3\}$ (recall that $T_3$ is the data we use in step 3).
Let $p_j$ be the probability of landing in bin $j$, given by $p_j = \prob(g(Z) \in I_j)$.
Since $\bins{}$ is 2-well-balanced, $p_j \geq \frac{1}{2B}$.
Since $n \geq c_B B\log{\frac{B}{\delta}}$, by the multiplicative Chernoff bound, for some large enough $c_B$, with probability at least $1 - \frac{\delta}{2}$, $|R_j| \geq \frac{p_j}{2}$.

Consider each bin $j$. Let $\mu_j$ be the expected output of $g$ in bin $j$, given by $\mu_j = \expect[g(Z) \; | \; g(Z) \in I_j]$. $\mu(R_j)$, the mean of the values in $R_j$, is the empirical average of $|R_j|$ such values, each bounded between $b_{j-1}$ and $b_j$ where $I_j = [b_{j-1}, b_j]$. So $\hat{\mu}(R_j)$ is sub-Gaussian with parameter:

\[ \sigma^2 = \frac{(b_j - b_{j-1})^2}{4|R_j|} \leq \frac{(b_j - b_{j-1})^2}{2p_jn} \]

Then by the sub-Gaussian tail bound, for any $1 \leq j \leq B$, with probability at least $1 - \frac{\delta}{2B}$, we have:
\begin{align} (\mu_j - \hat{\mu}(R_j))^2 \leq \frac{(b_j - b_{j-1})^2}{p_jn} \log{\frac{4B}{\delta}}\label{eqn:1} \end{align}

So by union bound with probability at least $1 - \frac{\delta}{2}$ the above holds for all $1 \leq j \leq B$ simultaneously.

We then bound the $\ell_2$-error.
\begin{align*}
||\hat{g_{\mathcal{B}}} - g_{\mathcal{B}}||_2 &= \sqrt{\sum_{j =1}^B p_j (\mu_j - \hat{\mu}(R_j))^2} \\
&\leq \sqrt{\sum_{j =1}^B p_j \frac{(b_j - b_{j-1})^2}{p_jn} \log{\frac{4B}{\delta}}} \tag{by equation~\eqref{eqn:1}}\\
&= \sqrt{\frac{1}{n} \log{\frac{4B}{\delta}} \sum_{j =1}^B (b_j - b_{j-1})^2 } \\
&\leq \sqrt{\frac{1}{n} \log{\frac{4B}{\delta}} \sum_{j =1}^B (b_j - b_{j-1}) } \tag{because $0\le b_j - b_{j-1}\le 1$}\\
&\leq \sqrt{\frac{1}{n} \log{\frac{4B}{\delta}} } \\
&\leq c_2 \frac{1}{\sqrt{n}} \sqrt{\log{\frac{B}{\delta}}}
\end{align*}

Similarly, we can also bound the $\ell_1$-error. Here we also use the fact that $p_j \leq \frac{2}{B}$ since $\bins{}$ is 2-well-balanced.
\begin{align*}
||\hat{g_{\mathcal{B}}} - g_{\mathcal{B}}||_1 &= \sum_{j =1}^B p_j |\mu_j - \hat{\mu}(R_j)| \\
&\leq \sum_{j =1}^B p_j \sqrt{\frac{(b_j - b_{j-1})^2}{p_jn} \log{\frac{4B}{\delta}}} \\
&=\sum_{j =1}^B \sqrt{\frac{p_j(b_j - b_{j-1})^2}{n} \log{\frac{4B}{\delta}}} \\
&\leq \sum_{j =1}^B \sqrt{\frac{2(b_j - b_{j-1})^2}{Bn} \log{\frac{4B}{\delta}}} \\
&\leq \sqrt{\frac{2}{Bn} \log{\frac{4B}{\delta}}} \sum_{j =1}^B (b_j - b_{j-1}) \\
&\leq c_1 \frac{1}{\sqrt{Bn}} \sqrt{\log{\frac{B}{\delta}}}
\end{align*}
By union bound, these hold with probability at least $1 - \delta$, which completes the proof.
\end{proof}

\textbf{Step 2}: Recall that we chose our bins so that each bin had an equal proportion of points in the recalibration set. In our proofs we required that this property (approximately) holds in the population as well. The following lemma shows this.

\begin{wellBalanced}
\wellBalancedText{}
\end{wellBalanced}

\begin{proof}
Suppose we are given a bin construction set of size $n$, $T_n = \{(z_1, y_1), \dots, (z_n, y_n)\}$.
For any interval $I$, let $\hat{P}(I)$ be the empirical estimate of $P(I) = \prob(g(Z) \in I)$ given by:
\[ \hat{P}(I) = \frac{|\{(z_i, y_i) \in T_n \mid g(z_i) \in I\}|}{n} \]

We constructed the bins so that each interval $I_j$ contains $\frac{n}{B}$ points, or in other words, $\hat{P}(I_j) = \frac{1}{B}$. We want to show that $\frac{1}{2B} \leq \prob(g(Z) \in I_j) \leq \frac{2}{B}$. Since the intervals are chosen from data, we want a uniform concentration result that holds for all such intervals $I_j$.

We will use a discretization argument. The idea is that we will cover $[0, 1]$ with $10B$ disjoint small intervals such that for each of these intervals $I_j'$, $P(g(Z) \in I_j') = \frac{1}{10B}$. We will then use Bernstein and union bound to get that with probability at least $1 - \delta$, for all $I_j'$, $|P(I_j') - \hat{P_j}(I_j')| \leq \frac{1}{100B}$ . Given an arbitrary interval $I$, we can write it as an approximate union of these small intervals, which will allow us to concentrate $|P(I) - \hat{P}(I)|$.

\textbf{Concentrating the small intervals:} Fix some interval $I_j'$. Let $w_i = \mathbb{I}(g(z_i) \in I_j')$ for $i = 1,\dots,n$. Then $w_i \sim \mbox{Bernoulli}(\frac{1}{10B})$. $\hat{P}(I_j')$ is simply the empirical average of $n$ such values and as such with probability at least $1 - \frac{\delta}{10B}$:
\[ |P(I_j') - \hat{P_j}(I_j')| \leq \sqrt{\frac{2}{10Bn} \log{\frac{10B}{\delta}}} + \frac{2}{3n} \log{\frac{10B}{\delta}} \]
If $n = cB \log{\frac{B}{\delta}}$ for a large enough constant $c$, we get:
\[ |P(I_j') - \hat{P_j}(I_j')| \leq \frac{1}{100B} \]
And this was with probability at least $1 - \frac{\delta}{10B}$. So by union bound we get that with probability at least $1 - \delta$ this holds for all $I_j'$.

\textbf{Concentrating arbitrary intervals:} Now consider arbitrary $I \subseteq [0, 1]$. We can approximately write $I$ as a union of the small $I_j'$ intervals. More concretely, we can form a lower bound for $\hat{P}(I)$ by considering all $I_j'$ contained in $I$:
\[ S_L = \{I_j' \mid I_j' \subseteq I \} \]
Similarly we can form an upper bound for $\hat{P}(I)$ by considering all $I_j'$ that have non-empty intersection with $I$:
\[ S_U = \{I_j' \mid I_j' \cap I \neq \emptyset \} \]
We can then show:
\[  \frac{9}{10} P(I) - \frac{1}{5B} \leq \hat{P}(I) \leq \frac{11}{10} P(I) + \frac{1}{5B} \]
Since in our case for all $j$, $\hat{P}(I_j) = \frac{1}{B}$, this gives us:
\[ \frac{1}{2B} \leq P(I_j) \leq \frac{2}{B} \]
\end{proof}

\textbf{Step 1}: Recall that step 1 essentially applies a scaling method---we fit a small number of parameters to the recalibration data.
We show that if $\G{}$ contains $g^* \in \G{}$ with low calibration error, then the empirical risk minimizer $g \in \G{}$ of the mean-squared loss will also quickly converge to low calibration error.
Intuitively, methods like Platt scaling fit a single parameter to the data so standard results in asymptotic statistics tell us they will converge quickly to their optimal error, at least in mean-squared error.
We can combine this with a decomposition of the mean-squared error into calibration and refinement, and the injectivity of $g \in \mathcal{G}$, to show they also converge quickly in calibration error.

\begin{lemma}[Convergence of scaling]
\label{lem:platt_scaling_bound}
Given $\delta$, there exists a constant $c$, such that for all $n$, $(\ce(g))^2 \leq \min_{g' \in \G{}}(\ce(g'))^2 + \frac{c}{n}$, with probability at least $1 - \delta$.
\end{lemma}

\begin{proof}

\textbf{From calibration error to mean-squared error}: We use the classic decomposition of the mean-squared error into calibration error (also known as reliability) and refinement\footnote{Note that the refinement term can be further decomposed into resolution (also known as sharpness) and irreducible uncertainty.}. For any $g \in \G{}$ we have:
\[ \mse(g) = \underbrace{(\ce(g))^2}_{\mbox{calibration}} + \underbrace{\expect[(\expect[Y \mid g(Z)] - Y)^2]}_{\mbox{refinement}} \]
Note that the refinement term is constant for all injective $g \in \G{}$, since for injective $g$:
\[ \expect[(\expect[Y \mid g(Z)] - Y)^2] = \expect[(\expect[Y \mid Z] - Y)^2] \]
This means that the difference in calibration error between any $g$ and $g'$ is precisely the difference in the mean-squared error. So it suffices to upper bound the generalization gap $\mse(g) - \mse(g^*)$ for the mean-squared error. Our analysis is fairly standard: we will show asymptotic convergence in the parameter space, and then use a Taylor expansion to show convergence in the MSE loss.

\textbf{Parameter convergence}: Recall that $\hat{\theta}$ denotes the parameters estimated by optimizing the empirical mean-squared error objective on $n$ samples in step 1 of our algorithm, and $\theta^*$ denotes the optimal parameters that minimize the mean-squared error objective on the population. From Theorem 5.23 of~\cite{vaart98asymptotic}, on the asymptotic parameter convergence of M-estimators, we have as $n \to \infty$:
\[ \sqrt{n}(\hat{\theta} - \theta^*) \to_D N(0, \Sigma) \]
Then for each $1 \leq i \leq d$, we have:
\[ \sqrt{n}(\hat{\theta}_i - \theta^*_i) \to_D N(0, \sigma_i^2) \]
We will show that there exists $c_i$ such that for each $i$ and for all $n$, with probability at least $1 - \frac{\delta}{d}$:
\[ \lvert \hat{\theta}_i - \theta^*_i \rvert \leq \frac{c_i}{n} \]
To see this, we begin with the definition of convergence in distribution, which says that the CDFs converge pointwise at every point where the CDF is continuous, which for a Gaussian is every point. That is, letting $z_i$ be a sample from $N(0, \sigma_i^2)$, we have for all $c$:
\[ \lim_{n \to \infty} \prob( \sqrt{n}(\hat{\theta}_i - \theta^*_i) \geq c) = \prob(z_i \geq c) \]
By considering the CDF at each point and its negative, we can show the same result for the absolute value:
\[ \lim_{n \to \infty} \prob( \sqrt{n} \lvert \hat{\theta}_i - \theta^*_i \rvert \geq c) = \prob(\lvert z_i \rvert \geq c) \]
The tails of a normal distribution are bounded, so we can choose $c_i'$ such that:
\[ \prob(\lvert z_i \rvert \geq c_i') \leq \frac{\delta}{2d} \]
By definition of limit, this means that we can choose $N_i$ such that for all $n \geq N_i$, we have:
\[ \prob( \sqrt{n} \lvert \hat{\theta}_i - \theta^*_i \rvert \geq c_i') \leq \frac{\delta}{d} \]
In other words, for all $n \geq N_i$, with probability at least $1 - \frac{\delta}{d}$:
\[ \lvert \hat{\theta}_i - \theta^*_i \rvert \leq \frac{c_i'}{\sqrt{n}} \]
Since this only does not hold for finitely many values $1, \cdots, N_i - 1$, we can `absorb' these cases into the constant. That is, for each $n \in \{1, \cdots, N_i - 1 \}$, there exists $r_n$ such that if we use $n$ samples, then except with probability $\frac{\delta}{d}$, $\lvert \hat{\theta}_i - \theta^*_i \rvert \leq r_n$. So then we can choose $c_i$ such that for all $n$:
\[ \lvert \hat{\theta}_i - \theta^*_i \rvert \leq \frac{c_i' + \max_{1 \leq m < N_i}{r_m \sqrt{m}}}{\sqrt{n}} \leq \frac{c_i}{\sqrt{n}} \]
We apply union bound over the indices $i$, and can then bound the $\ell_2$-norm of the difference between the estimated and optimal parameters, so that we can choose $k$ such that for all $n$, with probability at least $1 - \delta$:
\[ ||\hat{\theta} - \theta^*||_2^2 \leq \frac{k}{n} \]

\textbf{Loss convergence}: We denote the loss by $L$, defined as:
\[ L(\theta) = \mse(g_{\theta}) = \expect[ (Y - g_{\theta}(X))^2 ] \]
We approximate the loss $L$ by the first few terms of its Taylor expansion, which we denote by $\widetilde{L}$:
\[ \widetilde{L}(\theta) = L(\theta^*) + \nabla L(\theta^*)^T (\hat{\theta} - \theta^*) + (\hat{\theta} - \theta^*)^T \nabla^2 L(\theta^*) (\hat{\theta} - \theta^*) \]
We assumed that $L$ was twice differentiable with continuous second derivative, and that $\theta^*$ minimized the loss in an open set, so $\nabla L(\theta^*) = 0$, and we also have (see e.g. Theorem 3.3.18 in~\cite{hubbard1998vector}):
\[ \lim_{||\hat{\theta} - \theta^*||_2 \to 0} \frac{L(\hat{\theta}) - \widetilde{L}(\hat{\theta})}{||\hat{\theta} - \theta^*||_2^2} = 0 \]
By the definition of a limit if we fix $\epsilon > 0$, there exists $R > 0$ such that if $||\hat{\theta} - \theta^*||_2 \leq R$ then $L(\hat{\theta}) - \widetilde{L}(\hat{\theta}) \leq \epsilon ||\hat{\theta} - \theta^*||_2^2$. For some large enough $N_0$, if $n \geq N_0$, then with probability at least $1 - \delta$, $||\hat{\theta} - \theta^*||_2 \leq R$. As before, since this only does not hold for finitely many $N$, we can fold these cases into the constant so that there exists $\epsilon'$ such that for all $n$, $L(\hat{\theta}) - \widetilde{L}(\hat{\theta}) \leq \epsilon' ||\hat{\theta} - \theta^*||_2^2$ with probability at least $1 - \delta$. Plugging in $\widetilde{L}(\hat{\theta})$, we have:
\[ L(\hat{\theta}) - L(\theta^*) \leq (\hat{\theta} - \theta^*)^T \nabla^2 L(\theta^*) (\hat{\theta} - \theta^*) + \epsilon' ||\hat{\theta} - \theta^*||_2^2 \]
We can bound this by the operator norm of the Hessian, and then use the parameter convergence result:
\[ L(\hat{\theta}) - L(\theta^*) \leq (|| \nabla^2 L(\theta^*) ||_{op} + \epsilon') ||\hat{\theta} - \theta^*||_2^2 \leq \frac{c}{n} \]
which holds with probability at least $1 - \delta$, as desired.

\end{proof}

Finally, we have the tools to prove the main theorem:

\begin{proof}[Proof of Theorem~\ref{thm:final-calib}]
The proof pieces together Lemmas~\ref{lem:platt_scaling_bound},~\ref{thm:empirical-binning},~\ref{lem:well-balanced} and Proposition~\ref{prop:bin_low_bound}.

For any fixed $c_1 > 0$, there exists $c_1'$ such that if $n \geq c_1'\big(\frac{1}{\epsilon^2}\big)$, from Lemma~\ref{lem:platt_scaling_bound}, step 1 of our algorithm gives us $g$ with $(\ce(g))^2 \leq \min_{g' \in \G{}}(\ce(g'))^2 + c_1 \epsilon^2$, with probability at least $1 - \frac{\delta}{3}$.

Next, for universal constant $c_2$, if $n \geq c_2(B \log{\frac{B}{\delta}})$, from Lemma~\ref{lem:well-balanced}, step 2 chooses a 2-well-balanced binning scheme $\bins{}$ with probability at least $1 - \frac{\delta}{3}$.

From Proposition~\ref{prop:bin_low_bound}, $(\ce(g_{\bins{}}))^2 \leq (\ce(g))^2 \leq \min_{g' \in \G{}}(\ce(g'))^2 + c_1 \epsilon^2$. Then from Lemma~\ref{thm:empirical-binning}, for any $c_3 > 0$, there exists $c_3'$ such that if $n \geq c_3'\big(\frac{1}{\epsilon^2} \log{\frac{B}{\delta}}\big)$, step 3 gives us $\hat{g_{\bins{}}}$ with $||\hat{g_{\bins{}}} - g_{\bins{}}||_2 \leq c_3 \epsilon$ with probability at least $1 - \frac{\delta}{3}$. We want to say that since $\hat{g_{\bins{}}}$ is close to $g_{\bins{}}$ and $g_{\bins{}}$ has low calibration error, this must mean that $\hat{g_{\bins{}}}$ has low calibration error.

To do this we represent the ($\ell_2$) calibration error of any $g$ as the distance between $g$ and a perfectly recalibrated version of $g$. That is, we define the perfectly recalibrated version of $g$ as:
\[ \omega(g)(z) = \expect[ Y \mid g(Z) = z ] \]
Then for any $g$, we can write $\ce(g) = ||g - \omega(g)||_2$. By triangle inequality on the $\ell_2$ norm on functions, we have:
\[ ||\hat{g_{\bins{}}} - \omega(g_{\bins{}})||_2 \leq ||\hat{g_{\bins{}}} - g_{\bins{}}||_2 + ||g_{\bins{}} - \omega(g_{\bins{}})||_2 \leq c_3 \epsilon + \sqrt{\min_{g' \in \G{}}(\ce(g'))^2 + c_1 \epsilon^2} \]

Now the LHS is not quite the calibration error of $\hat{g_{\bins{}}}$, which is $||\hat{g_{\bins{}}} - \omega(\hat{g_{\bins{}}})||_2$~\footnote{This is a very technical point, so at a first pass the reader may skip the following discussion.}.
However, since $g$ is injective, $g_{\bins{}}$ takes on a different value for each interval $I_j \in \bins{}$.
If $\hat{g_{\bins{}}}$ also takes on a different value for each interval $I_j \in \bins{}$, then we can see that $\omega(g_{\bins{}}) = \omega(\hat{g_{\bins{}}})$.
If not, $\omega(\hat{g_{\bins{}}})$ can only merge some of the intervals of $\omega(g_{\bins{}})$, and by Jensen's we can show:
\[ ||\hat{g_{\bins{}}} - \omega(\hat{g_{\bins{}}})||_2 \leq ||\hat{g_{\bins{}}} - \omega(g_{\bins{}})||_2 \leq c_3 \epsilon + \sqrt{\min_{g' \in \G{}}(\ce(g'))^2 + c_1 \epsilon^2} \]
An alternative way to see this is to add infinitesimal noise to $\hat{g_{\bins{}}}$ for each interval $I_j$, in which case we get $\omega(g_{\bins{}}) = \omega(\hat{g_{\bins{}}})$.
Finally we convert back from $\ce{}$ to the squared calibration error:
\[ (\ce(\hat{g_{\bins{}}}))^2 = ||\hat{g_{\bins{}}} - \omega(\hat{g_{\bins{}}})||_2^2 \leq \min_{g' \in \G{}}(\ce(g'))^2 + (c_3^2 + c_1) \epsilon^2 + 2\sqrt{(c_3^2 \epsilon^2) \big(\min_{g' \in \G{}}(\ce(g'))^2 + c_1 \epsilon^2\big)} \]
By the AM-GM inequality, we have:
\[ 2\sqrt{(c_3^2 \epsilon^2) \big(\min_{g' \in \G{}}(\ce(g'))^2 + c_1 \epsilon^2\big)} \leq (c_3^2 + c_1) \epsilon^2 + \min_{g' \in \G{}}(\ce(g'))^2 \]
Combining these, we get:
\[ (\ce(\hat{g_{\bins{}}}))^2 \leq 2 \min_{g' \in \G{}}(\ce(g'))^2 + 2(c_3^2 + c_1) \epsilon^2 \]
By e.g. choosing $c_1 = 0.1$ and $c_3 = 0.1$, we have $2(c_3^2 + c_1) \leq 1$, which gives us the desired result. By union bound over each step, we have this with probability at least $1 - \delta$.

\end{proof}

\subsection{Bounding the mean-squared error}

We also show that if we use lots of bins, discretization has little impact on model quality as measured by the mean-squared error.
Note that recalibration itself typically \emph{reduces/improves} the mean-squared error.
However, in our method after fitting a recalibration function like Platt scaling does, we discretize the function outputs.
This reduces the calibration error and allows us to measure the calibration error, but it does increase the mean-squared error by a small amount.
Here we upper bound the increase in mean-squared error.
In other words, our method allows for the calibration error of the final model to be measured, and has little impact on the mean-squared error.

\begin{restatable}[MSE Bound]{proposition}{mseFiniteBinning}
\label{prop:mse-finite-binning}
If $\mathcal{B}$ is a 2-well-balanced binning scheme of size $B$ and $B = \widetilde{\Omega}(n)$, where $\widetilde{\Omega}$ hides $\log$ factors, then $\mse(\hat{g}_{\mathcal{B}}) \leq \mse(g) + O(\frac{1}{B})$.
\end{restatable}

To show this we begin with a lemma showing that if $f$ and $g$ are close in $\ell_1$ norm, then their mean-squared errors are close:

\begin{lemma}
\label{lem:mse-l1}
For $f, g : [0, 1] \to [0, 1]$, $\mse(f) \leq \mse(g) + 2||f - g||_1$.
\end{lemma}

\begin{proof}
\begin{align*}
\expect[(f(Z) - Y)^2 - (g(Z) - Y)^2] &= \expect[(f(Z) - g(Z))(f(Z) + g(Z) - 2Y)] \\
&\leq \expect[|f(Z) - g(Z)||f(Z) + g(Z) - 2Y|] \\
&\leq \expect[2|f(Z) - g(Z)|] \\
& =2||f-g||_1
\end{align*}
Where the third line followed because $-2 \leq f(Z) + g(Z) - 2Y \leq 2$.
\end{proof}

Next, we show that in the limit of infinite data, if we bin with a well-balanced binning scheme then the MSE cannot increase by much.

\begin{lemma}
\label{thm:bin-sharpness}
Let $\mathcal{B}$ be an $\alpha$-well-balanced binning scheme of size $B$. Then $\mse(g_{\mathcal{B}}) \leq \mse(g) + \frac{2\alpha}{B}$.
\end{lemma}

\begin{proof}
We bound $||g_{\mathcal{B}} - g||_1$ and then use Lemma~\ref{lem:mse-l1}. We use the law of total expectation, conditioning on $\beta(g(Z))$, the bin that $g(Z)$ falls into.
\begin{align*}
||g_{\mathcal{B}} - g||_1 &= \mathbb{E}[|g_{\mathcal{B}}(Z) - g(Z)|] \\
&\leq \mathop{\mathbb{E}}_{\beta(g(Z))} \Big[ \mathop{\mathbb{E}}_{Z | \beta(g(Z))} [ |g_{\mathcal{B}}(Z) - g(Z)| ]\Big]\\
&\leq \mathop{\mathbb{E}}_{\beta(g(Z))} \Big[ b_{\beta(g(Z))} - b_{\beta(g(Z))-1}\Big]
\end{align*}
We now use the fact that $\mathcal{B}$ is $\alpha$-well-balanced.
\begin{align*}
\mathop{\mathbb{E}}_{\beta(g(Z))} \Big[ (b_{\beta(g(Z))} - b_{\beta(g(Z))-1})\Big] &= \sum_{i=1}^B \prob\big(g(Z) \in [b_{\beta(g(Z))-1}, b_{\beta(g(Z))}]\big) (b_{\beta(g(Z))} - b_{\beta(g(Z))-1}) \\
&\leq \sum_{i=1}^B \frac{\alpha}{B} (b_{\beta(g(Z))} - b_{\beta(g(Z))-1}) \\
&\leq \frac{\alpha}{B}
\end{align*}
Finally, from Lemma~\ref{lem:mse-l1}, we get that $\mse(g_{\mathcal{B}}) \leq \mse(g) + \frac{2\alpha}{B}$.
\end{proof}

The above lemma bounds the increase in MSE due to binning in the infinite sample case -- next we deal with the finite sample case and prove proposition~\ref{prop:mse-finite-binning}:

\begin{proof}[Proof of Proposition~\ref{prop:mse-finite-binning}:]
Ignoring all $\log$ factors, from Theorem~\ref{thm:empirical-binning} if $n = \widetilde{\Omega}(B)$, we have $||\hat{g}_{\mathcal{B}} - g_{\mathcal{B}}||_1 = O(\frac{1}{\sqrt{nB}})$. Then from  Lemma~\ref{lem:mse-l1}, $\mse(\hat{g}_{\mathcal{B}}) \leq \mse(g_{\mathcal{B}}) + O(\frac{1}{\sqrt{Bn}}) \leq \mse(g_{\mathcal{B}}) + O(\frac{1}{B})$. From Theorem~\ref{thm:bin-sharpness}, since $\mathcal{B}$ is 2-well-balanced, we have  $\mse(g_{\mathcal{B}}) \leq \mse(g) + O(\frac{1}{B})$. This gives us $\mse(\hat{g}_{\mathcal{B}}) \leq \mse(g) + O(\frac{1}{B})$.
\end{proof}

\subsection{Alternative binning schemes}
\label{sec:alt_binning_schemes}

We note that there are alternative binning schemes in the literature.
For example, the $B$ bins can be chosen as $I_1 = [0, \frac{1}{B}], I_2 = (\frac{1}{B}, \frac{2}{B}], \dots, I_B = (\frac{B-1}{B}, 1]$.
The main problem with this binning scheme is that we may not be able to measure the calibration error efficiently, which is critical.
However, if we choose the bins like this, and are lucky that the binning scheme happens to be 2-well-balanced, we can improve the bounds on the MSE that we proved above.
This motivates alternative hybrid binning schemes, where we try to keep the width of the bins as close to $1/B$ as possible, while ensuring that each bin contains lots of points as well.
We think analyzing what binning schemes lead to the best bounds, and seeing if this can improve the calibration method, is a good direction for future research.

\section{Experimental details and ablations for section~\ref{sec:calibrating_models}}

\label{sec:calibrating_models_appendix_experiments}

We give more experimental details for our CIFAR-10 experiment, show experimental results for top-label calibration in ImageNet and CIFAR-10, and give details and results for our synthetic experiments. Note that the code is available in the supplementary folder for completeness.

\textbf{Experimental details}: We detail our experimental protocol for CIFAR-10 first. The CIFAR-10 validation set has 10,000 data points. We sampled, with replacement, a recalibration set of 1,000 points. In our theoretical approach and analysis, we split up these sets into multiple parts. For example, we used the first part for training a function, second part for bin construction, third part for binning. In practice, using the same set for all three steps worked out better, for both histogram binning and \ourcal{}. We believe that there may be theoretical justification for merging these sets, although we leave that for future work. For the marginal calibration experiment we ran either \ourcal{} (we fit a sigmoid in the function fitting step) or histogram binning. We calibrated each of the $K$ classes seprately as described in Section~\ref{sec:formulation}, and measured the marginal calibration error on the entire set of 10K points. We repeated this entire procedure 100 times, and computed mean and 90\% confidence intervals.

\begin{figure}
  \centering
  \centering
  	 \begin{subfigure}[b]{0.48\textwidth}
         \centering
         \includegraphics[width=\textwidth]{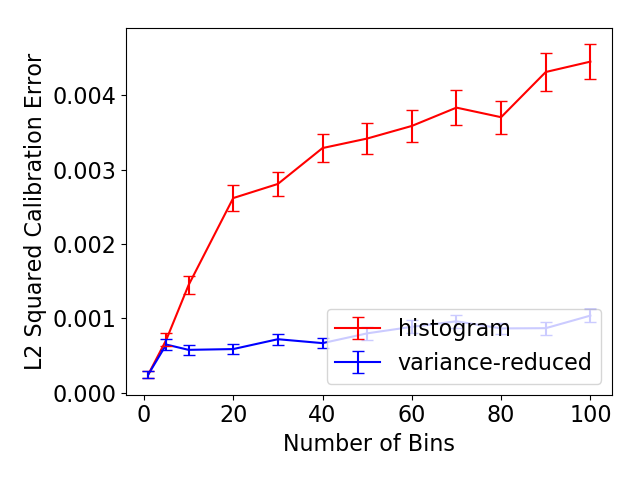}
         \caption{Effect of number of bins $B$ on top calibration error on ImageNet.
         }
         \label{fig:imagenet_top_cal_var_red}
     \end{subfigure}
     \hfill
     \begin{subfigure}[b]{0.48\textwidth}
         \centering
         \includegraphics[width=\textwidth]{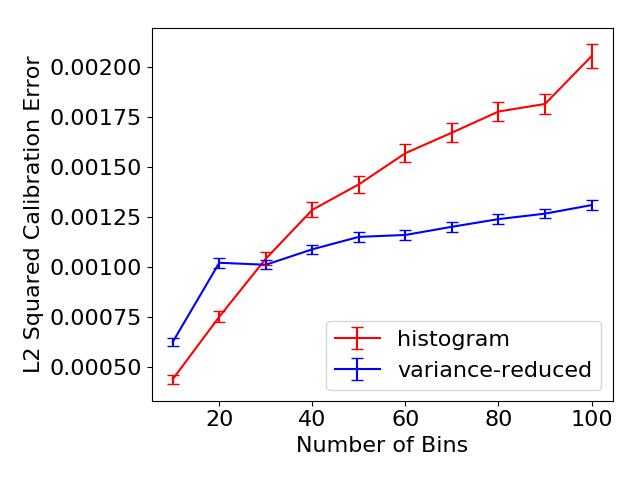}
         \caption{Effect of number of bins $B$ on top calibration error on CIFAR-10.
         }
         \label{fig:cifar_top_cal_var_red}
     \end{subfigure}
  \caption{
    Recalibrating using 1,000 data points on ImageNet and CIFAR-10, \ourcal{} typically achieves lower squared calibration error than histogram binning, especially when the number of bins $B$ is large. The difference is very significant on ImageNet, where our method does better when $B \geq 10$, and gets a nearly 5 times lower calibration error when $B = 100$. For CIFAR-10 our method does better when $B > 30$, which supports the theory, which predicts that our method does better when $B$ is large. However, when $B$ is small, practitioners should try both histogram binning and \ourcal{}.
    \pl{still need to update the legend to scaling-binning}
}
  \label{fig:mse_estimators_bins}
\end{figure}

In this experiment, we are checking a very precise hypothesis---assuming that the empirical distribution on the 10,000 validation points is the true data distribution, how do these methods perform? This is similar to the experimental protocol used in e.g.~\cite{brocker2012empirical}.
An alternative experimental protocol would have been to first split the CIFAR-10 data into two sets of size $(1000, 9000)$.
We could have then used the first set to recalibrate the model using either \ourcal{} or histogram binning, and then used the remaining 9,000 examples to estimate the calibration error on the ground truth distribution, using Bootstrap to compute confidence intervals.
However, when we ran this experiment, we noticed that the results were very sensitive to which set of 1,000 points we used to recalibrate.
Multiple runs of this experiment led to very different results.
The point is that there are two sources of randomness---the randomness in the data the recalibration method operates on, and the randomness in the data used to evaluate and compare the recalibrators.
In our protocol we account for both of these sources of randomness.

\textbf{Top-label calibration experiments}: We also ran experiments on top-label calibration, for both ImageNet and CIFAR-10. The protocol is exactly as described above, except instead of calibrating each of the $K$ classes, we calibrated the top probability prediction of the model. More concretely, for each input $x_i$, the uncalibrated model outputs a probability $p_i$ corresponding to its top prediction $k_i$, where the true label is $y_i$. We create a new dataset $\{(p_1, \mathbb{I}(k_1 = y_1)), \dots, (p_n, \mathbb{I}(k_n = y_n))\}$ and run \ourcal{} (fitting a sigmoid in the function fitting step, as in Platt scaling) or histogram binning on this dataset, using $B$ bins. This calibrates the probability corresponding to the top prediction of the model. We evaluate the recalibrated models on the top-label calibration error metric described in Section~\ref{sec:formulation}. For both CIFAR-10 and ImageNet we sampled, with replacement, a recalibration set of 1,000 points for the recalibration data, and we measured the calibration error on the entire set (10,000 points for CIFAR-10, and 50,000 points for ImageNet) as above. We show $90\%$ confidence intervals for all plots.

Figure~\ref{fig:imagenet_top_cal_var_red} shows that on ImageNet \ourcal{} gets significantly lower calibration errors than histogram binning when $B \geq 10$, and nearly a 5 times lower calibration error when $B = 100$. Both methods get similar calibration errors when $B = 1$ or $B = 5$. Figure~\ref{fig:cifar_top_cal_var_red} shows that on CIFAR-10 when $B$ is high, \ourcal{} gets lower calibration errors than histogram binning, but when $B$ is low histogram binning gets lower calibration errors. We believe that the difference might be because the CIFAR-10 model is highly accurate at top-label prediction to begin with, getting an accuracy of over $93\%$, so there is not much scope for re-calibration. In any case, this ablation tells us that practitioners should try multiple methods when recalibrating their models and evaluate their calibration error.

\textbf{(A) Synthetic experiments to validate bounds}: We first describe the scaling family we use, which is Platt scaling after applying a log-transform~\cite{platt1999probabilistic}, otherwise known as beta calibration~\cite{kull2017sigmoids}. Let $\sigma$ be the standard sigmoid function given by:
\[ \sigma(x) = \frac{1}{1 + \exp(-x)} \]
Then, our recalibration family $\mathcal{G}$ consists of $g$ parameterized by $a, c$, given by:
\[ g(z; a, c) = \sigma\Big( a\log{\frac{z}{1-z}} + c \Big) \]
In this set of synthetic experiments, we assume well-specification, that is $P(Y = 1 \mid Z=z) = g(z; a, c)$ for some $a, c$. We set $P(Z) = \mbox{Uniform}[0, 1]$. Since we know $P(Y = 1 \mid Z)$, we can approximate the true squared calibration error in this case, even for scaling methods. To do this, we sample $m=10000$ points $z_1, \dots, z_m$ independently from $P(Z)$. An \emph{unbiased} estimate of the squared calibration error then is:
\[ (\ce(g))^2 \approx \frac{1}{m} \sum_{i=1}^m \big[P(Y \mid Z = z_i) - g(z_i)\big]^2 \]
For each $n$ (number of recalibration samples) and $B$ (number of bins), we run either histogram binning or \ourcal{} with scaling family $\mathcal{G}$ and evaluate its calibration error as described above. We repeat this 1000 times, and compute 90\% confidence intervals. We fix $a = 2$ and $c = 1$.

\begin{figure}
  \centering
  \centering
     \begin{subfigure}[b]{0.48\textwidth}
         \centering
         \includegraphics[width=\textwidth]{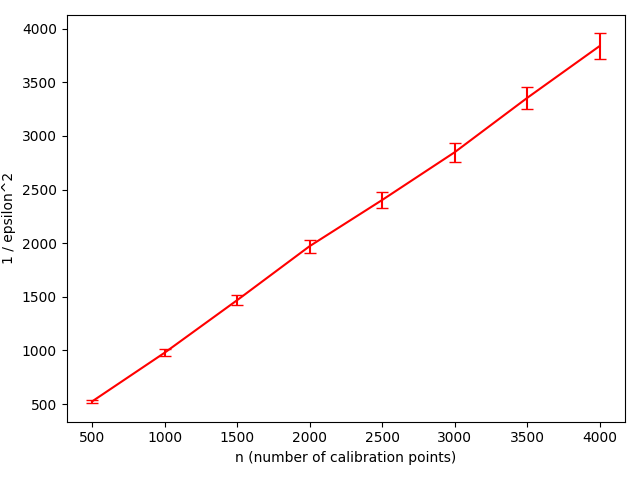}
         \caption{Histogram binning.
         }
         \label{fig:well-spec-vary-n-hist}
     \end{subfigure}
     \hfill
     \begin{subfigure}[b]{0.48\textwidth}
         \centering
         \includegraphics[width=\textwidth]{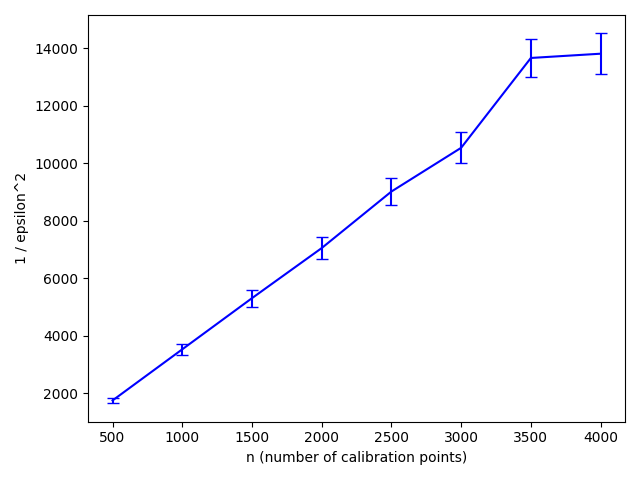}
         \caption{\Ourcal{}.
         }
         \label{fig:well-spec-vary-n-var-red}
     \end{subfigure}
  \caption{
    Plots of $1/\epsilon^2$ against $n$ (recall that $\epsilon^2$ is the squared calibration error). We see that for both methods $1/\epsilon^2$ increases approximately linearly with $n$, which match the theoretical bounds.
    \tnote{does it make sense to make this two plots have the same y limit? Or we could even put the two curves one the same figure?}
}
  \label{fig:well-spec-vary-n}
\end{figure}

In the first sub-experiment we fix $B = 10$ and vary $n$, plotting $1/\epsilon^2$ in Figure~\ref{fig:well-spec-vary-n} (recall that $\epsilon^2$ is the squared calibration error). We plot the calibration errors for each method in a different plot because of the difference in scales, \ourcal{} achieves a much lower calibration error than histogram binning. As the theory predicts, we see that $1/\epsilon^2$ is approximately linear in $n$ for both calibrators. For example, when $B=10$ if we increase from $n=1000$ to $n=2000$ the squared calibration error of histogram binning decreases by $2.00 \pm 0.06$ times, and the squared calibration error of our method decreases by $1.98 \pm 0.09$ times.

\begin{figure}
  \centering
  \centering
     \begin{subfigure}[b]{0.48\textwidth}
         \centering
         \includegraphics[width=\textwidth]{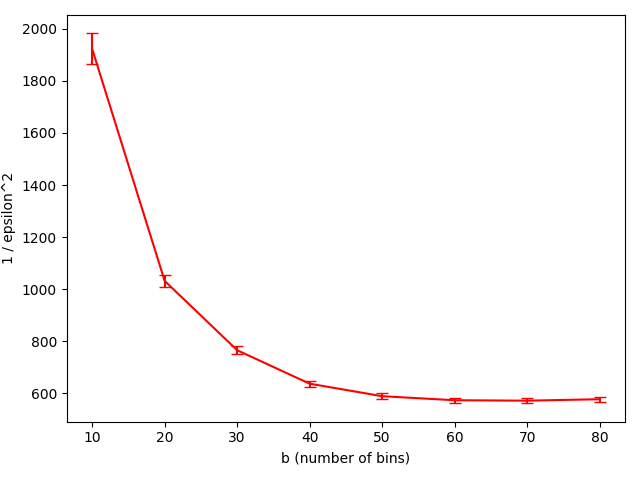}
         \caption{Histogram binning.
         }
         \label{fig:well-spec-vary-b-hist}
     \end{subfigure}
     \hfill
     \begin{subfigure}[b]{0.48\textwidth}
         \centering
         \includegraphics[width=\textwidth]{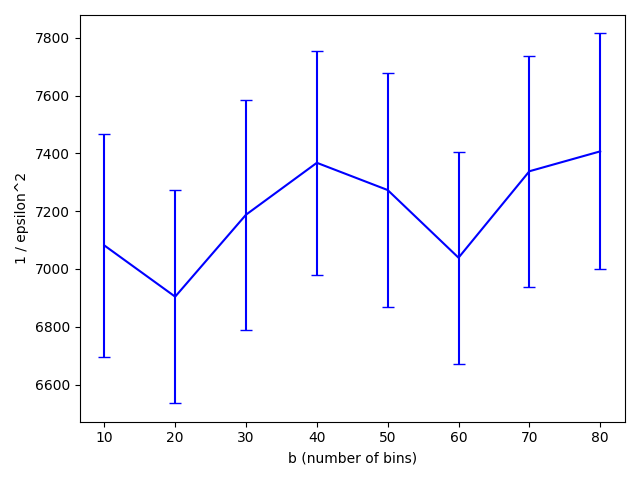}
         \caption{\Ourcal{}.
         }
         \label{fig:well-spec-vary-b-var-red}
     \end{subfigure}
  \caption{
    Plots of $1/\epsilon^2$ against $b$ (recall that $\epsilon^2$ is the squared calibration error). Note that the $Y$ axis for \ourcal{} is clipped to 6600 and 7800 to show the relevant region. We see that for histogram binning $1/\epsilon^2$ scales close to $1/B$, in other words the calibration error increases with the number of bins (important note: the plot decreases because we plot the inverse $1/\epsilon^2$). For \ourcal{} $1/\epsilon^2$ is relatively constant, within the margin of estimation error, as predicted by the theory.
    \tnote{similar comments to the those for the figure above}
}
  \label{fig:well-spec-vary-b}
\end{figure}

In the second sub-experiment we fix $n = 2000$ and vary $B$, plotting $1/\epsilon^2$ in Figure~\ref{fig:well-spec-vary-b}. For \ourcal{} $1/\epsilon^2$ is nearly constant (within the margin of error), but for histogram binning $1/\epsilon^2$ scales close to $1/B$. When $n = 2000$ and we increase from $5$ to $20$ bins, our method's squared calibration error decreases by $2\% \pm 7\%$ but for histogram binning it increases by $3.71 \pm 0.15$ \emph{times}. For reference, we plot $P(Y \mid Z = z)$ in Figure~\ref{fig:well-spec-curve}.

\begin{figure}
  \centering
  \centering
     \begin{subfigure}[b]{0.48\textwidth}
         \centering
         \includegraphics[width=\textwidth]{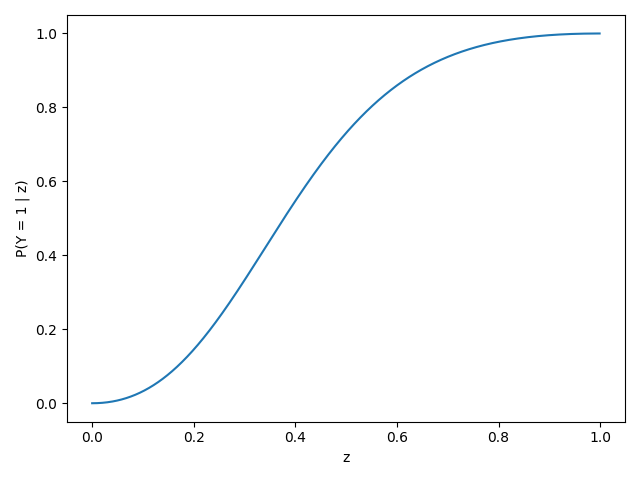}
         \caption{$P(Y \mid Z=z)$ for Experiment (A)
         }
         \label{fig:well-spec-curve}
     \end{subfigure}
     \hfill
     \begin{subfigure}[b]{0.48\textwidth}
         \centering
         \includegraphics[width=\textwidth]{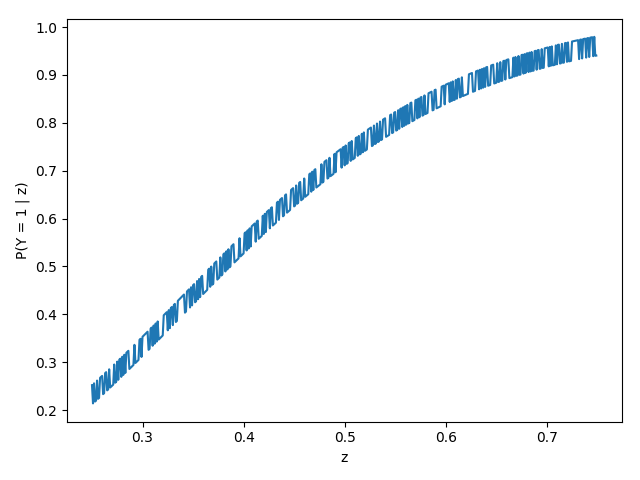}
         \caption{$P(Y \mid Z=z)$ for Experiment (B)
         }
         \label{fig:noisy-spec-curve}
     \end{subfigure}
  \caption{
    Plots of $P(Y \mid Z=z)$ against $z$ for both synthetic experiments.
}
  \label{fig:p_y_z_plots}
\end{figure}

\textbf{(B) Synthetic experiments to compare \ourcal{} and the scaling method}: We run an illustrative toy experiment to show that there are some cases where \ourcal{} does better than the underlying scaling method---there are other cases where the underlying scaling method does better. \ourcal{} can do better because if we have infinite data, Proposition~\ref{prop:bin_low_bound} showed that the binned version $g_{\bins{}}$ has lower calibration error than $g$. On the other hand, in step 3 of \ourcal{} algorithm we empirically bin the outputs of the scaling method which incurs an estimation error, and could mean \ourcal{} has higher calibration error than the underlying scaling method. Our key advantage is that unlike scaling methods our method has measurable calibration error so if we are not calibrated we can get more data or use a different scaling family.

Building on the previous synthetic experiments, in this experiment, we set the ground truth $P(Y = 1 \mid Z=z) = g(z; a, c) + h(z)$ where for each $z$, $h(z) \in \{-0.02, 0.02\}$ with equal probability. In this case we set $P(Z) = \mbox{Uniform}[0.25, 0.75]$ so that $P(Y = 1 \mid Z=z) \in [0, 1]$. We fix $B=10$ and vary $n$, plotting the squared calibration error $\epsilon^2$ in Figure~\ref{fig:well-spec-vary-b}. With $B=10$ bins, $n = 3000$ the squared calibration error is $5.2 \pm 1.1$ times lower for \ourcal{} than the underlying scaling method using a sigmoid recalibrator. For reference, we plot $P(Y \mid Z = z)$ in Figure~\ref{fig:noisy-spec-curve}.

\begin{figure}
  \centering
  \includegraphics[width=0.6\textwidth]{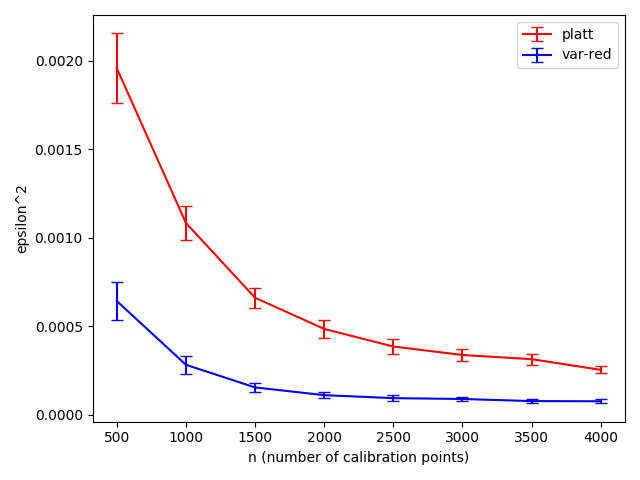}
  \caption{Plot of $\epsilon^2$ (squared calibration error) against number of samples $n$ used to recalibrate. We can see in this case \ourcal{} consistently gets lower calibration error.
  }
  \label{fig:well-spec-vary-b}
\end{figure}

\newpage

\section{Proofs for section~\ref{sec:verifying_calibration}}
\label{sec:verifying_calibration_appendix}

In this section we prove the finite sample bounds for the plugin and debiased estimators. We follow a very similar structure for both the plugin estimator and the debiased estimators.

We first give a proof for the plugin estimator. At a high level, we decompose the plugin estimator into three terms (Lemma~\ref{lem:plugin_decomp}), and then bound each of these terms. Most of these terms simply involve algebraic manipulation and standard concentration results, except Lemma~\ref{lem:p2_bound} which requires some tricky conditioning.

The debiased estimator decomposes into three terms as well, two of these terms are the same as those in the plugin estimator. Bounding the third term (Lemma~\ref{lem:c3_bound}) is the key to the improved sample complexity of the plugin estimator. The debiased estimator is not completely unbiased. However, with high probability if we condition on the $x_i$s in the evaluation set, each of these error terms is unbiased. We can then use Hoeffding's to concentrate each term near 0. The errors in each bin are then independent which leads to some cancelations of the error terms when we sum them up.

\textbf{We use the following notation simplification} to simplify the theorem statements and proofs:
\[ p_i = P(f(X) = s_i) \]
\[ y_i^* = \mathbb{E}[Y \; | \; f(X) = s_i] \]
\[ e_i = (s_i - y_i^*) \]

Then, if we let ${E^*}^2$ denote the actual squared calibration error, we have:
\[ {E^*}^2 = \sum_{i=1}^b p_i e_i^2 \]


We begin by noting that $\hat{p_i}$ is close to $p_i$ for all $i$. This is a standard application of either Bernstein's inequality or the multiplicative Chernoff bound.

\begin{lemma}
\label{lem:pi_bound}
Suppose $p_i > \piSmallBound{}$ for all $i$. Then we can define $c(n) < 0.5$ such that except with probability $\delta$ for all $i$ we have:
\[ |\hat{p_i} - p_i| < c(n)p_i := \sqrt{\frac{3}{n \min p_i} \log{\frac{2B}{\delta}}} p_i \]
\end{lemma}

\subsection{Analysis of plugin estimator (proof of Theorem~\ref{thm:final-plugin})}

The following lemma is crucial -- we decompose the plugin estimator into three terms that we can bound separately.

\begin{lemma}[Plugin decomposition]
\label{lem:plugin_decomp}
The plugin estimator satisfies the following decomposition:
\[ \pluginEst{} = \underbrace{\sum_{i=1}^b \hat{p_i}e_i^2}_{(P1)}  - \underbrace{2\sum_{i=1}^b \hat{p_i}e_i(\hat{y_i} - y_i^*)}_{(P2)} + \underbrace{\sum_{i=1}^b \hat{p_i}(\hat{y_i} - y_i^*)^2}_{(P3)} \]
\end{lemma}

\begin{proof}
The proof is by algebra:
\begin{align*}
\pluginEst{} &= \sum_{i=1}^b \hat{p_i}(s_i - \hat{y_i})^2 \\
&= \sum_{i=1}^b \hat{p_i}[e_i - (\hat{y_i} - y_i^*)]^2 \\
&= \underbrace{\sum_{i=1}^b \hat{p_i}e_i^2}_{(P1)}  - \underbrace{2\sum_{i=1}^b \hat{p_i}e_i(\hat{y_i} - y_i^*)}_{(P2)} + \underbrace{\sum_{i=1}^b \hat{p_i}(\hat{y_i} - y_i^*)^2}_{(P3)}
\end{align*}
\end{proof}

We now bound each of these three terms with the following three lemmas. We condition on $|\hat{p_i} - p_i| < c(n)p_i < 0.5p_i$ for all $i$, which holds with high probability from Lemma~\ref{lem:pi_bound}.

\tm{In general, writing proofs is like writing code. lemmas and theorems have type checked inputs and outputs. Inputs are conditiond and definitions, outputs ar conclusions. }
\begin{lemma}
\label{lem:p1_bound}
Let $(P1)$ be as defined in Lemma~\ref{lem:plugin_decomp}.
Suppose $|\hat{p_i} - p_i| < c(n) p_i$ for all $i$. Then
\[ |(P1) - {E^*}^2| \leq c(n) {E^*}^2 \]
\end{lemma}

\begin{proof}
The proof is by algebra. 
\begin{align*}
|(P1) - {E^*}^2| &= | \sum_{i=1}^b \hat{p_i}e_i^2 - \sum_{i=1}^b p_i e_i^2 | \\
&= | \sum_{i=1}^b (\hat{p_i} - p_i) e_i^2 | \\
&\leq \sum_{i=1}^b |(\hat{p_i} - p_i)| e_i^2 \\
&\leq \sum_{i=1}^b c(n) p_i e_i^2 \\
&\leq c(n) \sum_{i=1}^b p_i e_i^2 \\
&\leq c(n) {E^*}^2
\end{align*}
\end{proof}

\begin{lemma}
\label{lem:p2_bound}
Let $(P2)$ be as defined in Lemma~\ref{lem:plugin_decomp}.
Suppose $|\hat{p_i} - p_i| < c(n)p_i < 0.5p_i$ for all $i$. Then with probability $\geq 1 - \delta$:
\[ |(P2)| \leq \sqrt{\frac{2(1+c(n)){E^*}^2}{n} \log{\frac{2}{\delta}}} \]
\end{lemma}

\begin{proof}
Recall that we evaluated our estimators on an independent and identically distributed evaluation set $T_n = \{(x_1, y_1), \dots, (x_n, y_n)\}$. Also, note that since $|\hat{p_i} - p_i| < p_i$, $\hat{p_i} > 0$. Let $Z = (f(x_1), \cdots, f(x_n))$ be a random variable. 

$\hat{y_i}$ simply takes the empirical average of the label values, and is therefore an unbiased estimator of $y_i^*$ even if we condition on $Z$:
\[ \expect[\hat{y_i} - y_i^* \mid Z] = 0 \]

Next we look at the distribution of $\hat{y_i} - y_i^* \mid Z$. For all $(x_j, y_j) \in T_n$, $y_j \in \{0, 1\}$. Additionally, $\{y_j \mid (x_j, y_j) \in T_n\} \mid Z$ is also independently (but not identically) distributed. So by Hoeffding's lemma, $\hat{y_i} - y_i^* \mid Z$ is sub-Gaussian with parameter $\frac{1}{4 \hat{p_i} n}$.

Here, we note that $\hat{p_i}$ is a constant given $Z$.
Then, we get that $\hat{p_i}e_i(\hat{y_i} - y_i^*) \mid Z$ has expected value $0$ and is sub-Gaussian with parameter:
\[ \sigma_i^2 = \hat{p_i}^2 e_i^2 \frac{1}{4 \hat{p_i} n} = \frac{\hat{p_i} e_i^2}{4n} \]
This means that the sum, $(P2)$ has expected value $0$ and is sub-Gaussian with  parameter:
\[ \sigma^2 = 2^2 \sum_{i=1}^B \sigma_i^2 = 4 \sum_{i=1}^B \frac{\hat{p_i} e_i^2}{4n} \leq \frac{(1+c(n)){E^*}^2}{n} \]
By applying the sub-Gaussian tail inequality, we get that with probability at least $1-\delta$,
\[ |(P2)| \leq \sqrt{\frac{2(1+c(n)){E^*}^2}{n} \log{\frac{2}{\delta}}} \]
Since this was true for all $Z$, this is true if we marginalize over $Z$ as well, which completes the proof.
\end{proof}

\begin{lemma}
\label{lem:p3_bound}
Let $(P3)$ be as defined in Lemma~\ref{lem:plugin_decomp}.
Suppose $|\hat{p_i} - p_i| < c(n)p_i < 0.5p_i$ for all $i$. Then with probability $\geq 1 - \delta$:
\[ |(P3)| \leq \frac{B}{2n} \log{\frac{2B}{\delta}} \]
\end{lemma}

\begin{proof}
Fix arbitrary $\hat{p_i}$s satisfying $|\hat{p_i} - p_i| < c(n) p_i < 0.5p_i$. Note that this gives us $\hat{p_i} > 0$.

By Hoeffding's bound, for any fixed $i$, with probability at least $1-\frac{\delta}{b}$:
\[ |\hat{y_i} - y_i^*| \leq \sqrt{\frac{1}{2\hat{p_i}n} \log{\frac{2B}{\delta}}} \]
Applying union bound over $i = 1, \cdots, B$, we get that the above holds \emph{for all $i$} with probability at least $1 - \delta$. Then with probability at least $1 - \delta$, for all $i$:
\[ | \hat{p_i} (\hat{y_i} - y_i^*)^2 | \leq \frac{1}{2n} \log{\frac{2B}{\delta}} \]
Summing over the bins $i = 1, \cdots, B$, we get:
\[ 0 \leq (P3) \leq \frac{B}{2n} \log{\frac{2B}{\delta}}  \]
\end{proof}

Bounding the error of the plugin estimator simply involves combining the bounds for each of the terms, $P1$, $P2$, $P3$.

\begin{restatable}{theorem}{pluginBound}
\label{thm:plugin-bound}
Let $p_i = P(f(X) = s_i)$ and suppose $p_i > \piSmallBound{}$ for all $i$. Let $c(n)$ be defined as:
\[ c(n) = \sqrt{\frac{3}{n \min p_i} \log{\frac{2B}{\delta}}} \]
Then for the plugin estimator, with probability at least $1 - 3\delta$,
\[ | \hat{E_{pl}}^2 - {E^*}^2 | \leq c(n){E^*}^2 + \sqrt{\frac{2(1+c(n)){E^*}^2}{n} \log{\frac{2}{\delta}}} + \frac{B}{2n} \log{\frac{2B}{\delta}} \]
\end{restatable}

\begin{proof}
We have:
\[ |\pluginEst{} - {E^*}^2| \leq |P1 - {E^*}^2| + |P2| + |P3| \]
From Lemma~\ref{lem:pi_bound} we have $|\hat{p_i} - p_i| < c(n)p_i < 0.5p_i$ with probability $\geq 1 - \delta$. Conditioning on this, we combine Lemmas~\ref{lem:p1_bound},~\ref{lem:p2_bound},~\ref{lem:p3_bound} with union bound to get the desired result.
\end{proof}

We then prove the final bound for the plugin estimator, which we recall below.

\begin{finalPlugin}
\finalPluginText{}
\end{finalPlugin}

\begin{proof}
This is now a simple corollary of Theorem~\ref{thm:plugin-bound}. For large enough constant $c$, where $c$ is a constant independent of all the other variables, we choose $n = c\frac{B}{\epsilon^2} \log{\frac{B}{\epsilon^2}}$. Plugging it into the bound of Theorem~\ref{thm:plugin-bound}, we get the desired result, that $\lvert \hat{E}^2 - {E^*}^2 \rvert \leq \frac{1}{2}{E^*}^2$. Notice that the dominating term is term $(P3)$ in Theorem~\ref{thm:plugin-bound}---we will see that the debiased estimator improves on this.

In fact, we can also show that in the worst case the plugin estimator will need at least $O(\frac{B}{\epsilon^2})$ samples to estimate the calibration error. To see this, first note that the bias of the plugin estimator, which comes from term $(P3)$ is at least $\frac{B}{n}$. Furthermore, in the analysis of the debiased estimator we show that the variance of this term is on the order of $O(\frac{\sqrt{B}}{n})$. So if $n < 0.1\frac{B}{\epsilon^2}$ we can consider very large $B$, and use Chebyshev to show that that with high probability the estimation error is larger than $\epsilon^2$. 

\end{proof}

\subsection{Analysis of debiased estimator (proof of Theorem~\ref{thm:final-ours})}

Next, we bound the error of the debiased estimator. The proof follows along the lines of the plugin estimator. We begin with a decomposition (Lemma~\ref{lem:canceling_decomp}), similar to the decomposition of the plugin estimator. However, one of the terms in the decomposition, $C3$, is different. Lemma~\ref{lem:c3_bound} bounds this term $C3$. The rest of the proof is the same as for the plugin estimator, so we omit the other proofs.

As with the plugin estimator, we have a decomposition for the debiased estimator.

\begin{lemma}[Debiased decomposition]
\label{lem:canceling_decomp}
The debiased estimator satisfies the following decomposition:
\[ \debiasedEst{} = \underbrace{\sum_{i=1}^b \hat{p_i}e_i^2}_{(C1)}  - \underbrace{2\sum_{i=1}^b \hat{p_i}e_i(\hat{y_i} - y_i^*)}_{(C2)} + \underbrace{\sum_{i=1}^b \hat{p_i}\Big[ (\hat{y_i} - y_i^*)^2 - \frac{\hat{y_i}(1 - \hat{y_i})}{\hat{p_i}n-1} \Big]}_{(C3)} \]
\end{lemma}

As with the plugin estimator, we bound each of the three terms. Notice that $C1$ and $C2$ are the same as terms $P1$ and $P2$ in the plugin estimator decomposition, so the bounds for those carry over. The next lemma bounds the error in $C3$.

\begin{lemma}
\label{lem:c3_bound}
Let $(C3)$ be as defined in Lemma~\ref{lem:canceling_decomp}.
Suppose $|\hat{p_i} - p_i| < c(n) p_i < 0.5 p_i$ for all $i$. Then with probability $\geq 1 - \delta$:
\[ |(C3)| \leq \frac{3\sqrt{B}}{n} \log{\frac{n}{\delta}} + \frac{\delta}{n} \]
\end{lemma}

\begin{proof}
Let $Z = (f(x_1), \cdots, f(x_n))$ be a random variable. We note that for all $i$, $\hat{p_i}$ is a deterministic function of $Z$. For convenience, define $t_i$ as follows:
\[ t_i = (\hat{y_i} - y_i^*)^2 - \frac{\hat{y_i}(1 - \hat{y_i})}{\hat{p_i}n-1} \]

\textbf{Computing the expectation:} The debiased estimator debiases the plugin estimator. In particular, we briefly explain why $\expect[C3 \mid Z] = 0$. Since $\hat{y_i}$ is the mean of $n\hat{p_i}$ draws of a Bernoulli with parameter $y_i^*$, we have:
\[ \expect[(\hat{y_i} - y_i^*)^2 \mid Z] = \frac{y_i^*(1 - y_i^*)}{n\hat{p_i}} \] 
The term we subtracted is the unbiased estimate of the standard deviation of the samples, so from elementary statistics:
\[ \expect\Big[\frac{\hat{y_i}(1 - \hat{y_i})}{\hat{p_i}n-1} \mid Z\Big] = \frac{y_i^*(1 - y_i^*)}{n\hat{p_i}} \]

Which implies that $\expect[C3 \mid Z] = 0$.

\textbf{Bounding each term:} By Hoeffding's bound, for any fixed $i$, we get that with probability at least $1 - \frac{\delta}{n}$: 
\[ |\hat{y_i} - y_i^*| \leq \sqrt{\frac{1}{2\hat{p_i}n} \log{\frac{2n}{\delta}}} \]

Let $E_i$ be the event that this is indeed the case. Condition on $E_i$ holding for all $i$ -- by union bound this happens with probability at least $1 - \delta$.
With some algebra, we then get:
\[ \lvert \hat{p_i}t_i \rvert = \Big\lvert \hat{p_i}\Big[ (\hat{y_i} - y_i^*)^2 - \frac{\hat{y_i}(1 - \hat{y_i})}{\hat{p_i}n-1} \Big] \Big\rvert \leq \frac{3}{2n} \log{\frac{B}{\delta}} \]

\textbf{Concentration:} Next, we analyze the concentration of $T = \big[(C3) \mid Z, \forall i. E_i\big]$ around its mean $\mu$. $\lvert \hat{p_i}t_i \rvert$ is bounded so is sub-Gaussian with parameter:
\[ \sigma_i^2 = \frac{9}{4n^2}\log{\frac{B}{\delta}} \]
Each term $\hat{p_i}t_i$ in the sum is independent, even when conditioned on $Z$.
So $T$ is sub-Gaussian with parameter:
\[ \sigma^2 = \sum_{i=1}^B \sigma_i^2 = \frac{9B}{4n^2}\log{\frac{B}{\delta}} \]
So by the sub-Gaussian tail bound, we have:
\[ \lvert T - \mu \rvert \leq \sqrt{2\sigma^2\log{\frac{1}{\delta}}} \leq \frac{3\sqrt{2}}{2} \frac{\sqrt{B}}{n} \sqrt{\log{\frac{n}{\delta}} \log{\frac{1}{\delta}}} \]
This can be simplified to:
\[ \lvert T - \mu \rvert \leq \frac{3\sqrt{B}}{n} \log{\frac{n}{\delta}} \]

\textbf{Bounding the bias:} Although $\expect[C3 \mid Z] = 0$, conditioning on $E_i$ introduces some bias.
However, we can show this bias is small. First, notice that $|t_i| \leq 1$. The event $E_i$ holds with probability at least $1 - \frac{\delta}{n}$. Then by the law of total expectation, conditioning on $E_i$ shifts the mean by at most $\frac{\delta}{n}$ -- in other words $|\expect[t_i \mid E_i, Z]| \leq \frac{\delta}{n}$.
Summing over $t_i$s, we get:
\[ \lvert \expect[(C3) \mid Z, \forall i. E_i] \rvert \leq \sum_{i=1}^B \hat{p_i} \lvert \expect[t_i \mid E_i, Z] \rvert \leq \frac{\delta}{n} \]

\textbf{Finishing up:} Combining the bias and concentration, we get that with probability at least $1 - 2\delta$:
\[ |(C3)| \leq \frac{3\sqrt{B}}{n} \log{\frac{n}{\delta}} + \frac{\delta}{n}\]

\end{proof}

\tm{where the theorem above is used. It sounds atypical to have a theorem in appendix..}
\tm{it's almost always better to write theorem first and then lemmas.. }
\ak{The main theorem is the main text, should I restate it at the start of this section? In some sense what I presented in the main text is a corollary of the theorem below, but it sounds weird to give a corollary without a theorem in the main text. And this theorem is a bit too unwieldy to put in the main text.}
We combine the bounds for $(C1)$, $(C2)$, $(C3)$, as in Theorem~\ref{thm:plugin-bound}, to bound the estimation error of the debiased estimator.

\begin{restatable}{theorem}{cancelingBound}
\label{thm:our-bound}
In the same setting as Theorem~\ref{thm:plugin-bound}, for the debiased estimator, with probability at least $1 - 4\delta$,
\[ | \hat{E}^2 - {E^*}^2 | \leq c(n){E^*}^2 + \sqrt{\frac{2(1+c(n)){E^*}^2}{n} \log{\frac{2}{\delta}}} + \frac{3\sqrt{B}}{n} \log{\frac{n}{\delta}} + \frac{\delta}{n}\]
\end{restatable}

We interpret the bound in Theorem~\ref{thm:our-bound} in two regimes. In the first regime, we fix the problem parameters $p_i, {E^*}^2$, and look at what happens as we send $n$ to infinity. In that case, the second term dominates, and we see that the error is approximately proportional to $\frac{1}{\sqrt{n}}$, which is the same as for the plugin estimator. However, in general we do not need to estimate the calibration error extremely finely, and may be satisfied as long as we estimate the calibration error within a constant multiplicative factor. That is, we might only need $n$ to be large enough so that our estimate $\hat{E}^2$ is on the right order, e.g. between $0.5 {E^*}^2$ and $1.5 {E^*}^2$ (where $0.5$ and $1.5$ can be replaced by other constants). In that regime, the third term dominates and the error is approximately proportional to $\frac{\sqrt{B}}{n}$, which is better than for the plugin estimator where it is proportional to $\frac{B}{n}$ (see Theorem~\ref{thm:plugin-bound}). This is captured in the final bound, where the proof closely parallels that of Theorem~\ref{thm:final-plugin}.

\begin{finalCanceling}
\finalCancelingText{}
\end{finalCanceling}

\newpage

\newcommand{\lonece}[0]{\ensuremath{\ell_1\textup{-CE}}}
\newcommand{\loneerror}[0]{\ensuremath{\mathcal{E}}}
\newcommand{\pluginLoneEst}[0]{\ensuremath{\hat{\mathcal{E}}_{\textup{pl}}}}
\newcommand{\debiasedLoneEst}[0]{\ensuremath{\hat{\mathcal{E}}_{\textup{db}}}}

\section{Additional experiments for section~\ref{sec:verifying_calibration}}
\label{sec:verifying_calibration_appendix_experiments}

\subsection{Debiasing the ECE}
\label{sec:debiasing_ece_experiments}

We propose a way to more accurately estimate the $\ell_1$ calibration error (popularly known as ECE), and run experiments on ImageNet and CIFAR-10 that show that we can estimate the error much better than prior work, which uses the plugin estimator. The key insight is the same as for the $\ell_2$ calibration error---the plugin estimator for the $\ell_1$ calibration error is biased and this bias leads to inaccurate estimates. To estimate the error better we can subtract an approximation of the bias which leads to a better estimate. The main difference is that for the $\ell_1$ calibration error we were not able to approximate the bias with a closed form expression and instead use a Gaussian approximation.

Recall that the $\lonece{}$ is the $\ell_p{}$ calibration error with $p=1$, redefined below:

\begin{definition}
The $\ell_1$ calibration error of $f : \mathcal{X} \to [0, 1]$ is given by:
\begin{align}
\lonece(f) = \expect\big[ \left|f(X) - \expect[Y \mid f(X)] \right| \big]
\end{align}
\end{definition}

Estimating the $\lonece{}$ for many models is challenging (see Section~\ref{sec:challenges-measuring}) so prior work instead selects a binning scheme $\bins{}$ and estimates $\lonece(f_{\bins{}})$ of a model $f$. Suppose we wish to measure the binned calibration error $\loneerror{} = \lonece(f_{\bins{}})$ of a model $f : \mathcal{X} \to [0, 1]$ where $|\bins{}| = B$. Suppose we get an evaluation set $T_n = \{(x_1, y_1), \dots, (x_n, y_n)\}$. Past work typically estimates the $\ell_1$ calibration error using a plugin estimate for each term:

\begin{definition}[Plugin estimator for $\lonece{}$]
  Let $L_k$ denote the data points where the model outputs a prediction in the $k$-th bin of $\bins{}$: $L_k = \{ (x_j, y_j) \in T_n \; | \; f(x_j) \in I_k \}$.
  
  Let $\hat{p}_k$ be the estimated probability of $f$ outputting a prediction in the $k$-th bin:
$\hat{p}_k = \frac{|L_k|}{n}$.

Let $\hat y_k$ be the empirical average of $Y$ in the $k$-th bin: $\hat y_k = \sum_{x, y \in L_k} \frac{y}{|L_k|}$.

Let $\hat s_k$ be the empirical average of the model outputs in the $k$-th bin: $\hat s_k = \sum_{x, y \in L_k} \frac{f(x)}{|L_k|}$.
  
  The plugin estimate for the binned $\ell_1$ calibration error is the weighted squared difference between $\hat y_k$ and $\hat s_k$:
\[ \pluginLoneEst{} = \sum_{k=1}^B \hat{p}_k \lvert \hat s_k - \hat y_k \rvert \]
\end{definition}

The plugin estimate is a biased estimate of the binned calibration error. Intuitively, this is because of the absolute value: on any finite samples $s_k$ and $y_k$ will differ and the absolute value of the difference will be positive, even if the population values are the same. More concretely consider a model $f$ where $\expect[f(X) \mid f(X) \in I_k] = \expect[Y \mid f(X) \in I_k]$ in every bin $k$. In that case the binned calibration error $\loneerror{}$ is $0$. But on any finite samples the plugin estimate $\pluginLoneEst{}$ will be larger than $0$. In particular, the plugin estimator overestimates the binned calibration error, and the extent of overestimation may be different for different models.

To improve the estimate, we can subtract an approximation of the bias. That is, we would like to output $\pluginLoneEst{} - (\expect[\pluginLoneEst{}] - \loneerror{})$ as our estimate of the calibration error, where $\expect[\pluginLoneEst{}] - \loneerror{}$ is the bias. However, $\expect[\pluginLoneEst{}] - \loneerror{}$ is difficult to approximate in closed form. Instead, we propose approximating it by simulating draws from a normal approximation. More precisely, let $y_k = \expect[Y \mid f(X) \in I_k]$. Then, each label in the $k$-th bin is a draw from a Bernoulli distribution with parameter $y_k$. So $\hat y_k$ is the mean of $n \hat p_k$ Bernoulli draws. Assuming that the number of points in each bin is not too small, we can approximate $\hat y_k$ using a Gaussian approximation, and use that to approximate the bias $\expect[\pluginLoneEst{}] - \loneerror{}$.

\begin{definition}[Debiased estimator for $\lonece{}$]
For each $k$, let $R_k$ be a random variable sampled from a normal approximation of the label distribution in the $k$-th bin: $R_k \sim N(\hat y_k, \frac{\hat y_k (1 - \hat y_k)}{n \hat{p}_k})$. The debiased estimate for the binned $\ell_1$ calibration error subtracts an approximation of the bias from the plugin estimate:
\[ \debiasedLoneEst{} = \pluginLoneEst{} - (\expect\Big[\sum_{k=1}^B \hat{p}_k \lvert \hat s_k - R_k \rvert \Big] - \pluginLoneEst{}) \]
\end{definition}

We can approximate the expectation in the debiased estimator by simulating many draws from a normal distribution, which is computationally fairly inexpensive. Note that our proposed estimator is a heuristic approach, and future work should examine whether we can get provably better estimation rates for estimating the $\ell_1$ calibration error, as we did for the $\ell_2$ calibration error. That might involve analyzing our proposed estimator, or may involve coming up with a completely different estimator.

\begin{figure}
  \centering
  \centering
     \begin{subfigure}[b]{0.45\textwidth}
         \centering
         \includegraphics[width=\textwidth]{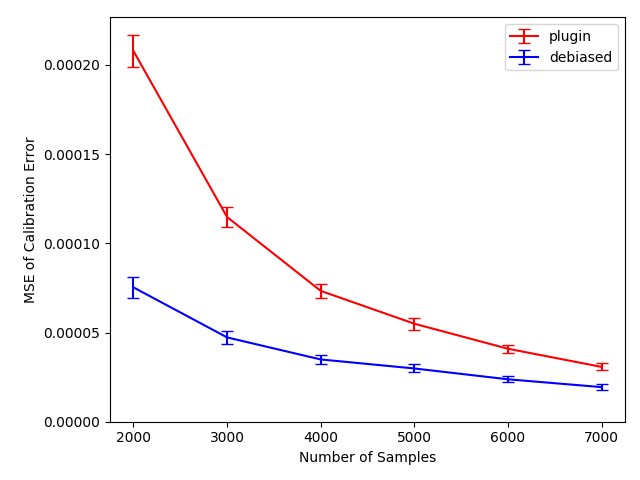}
         \caption{$B = 15$
         }
     \end{subfigure}
     \hfill
     \begin{subfigure}[b]{0.45\textwidth}
         \centering
         \includegraphics[width=\textwidth]{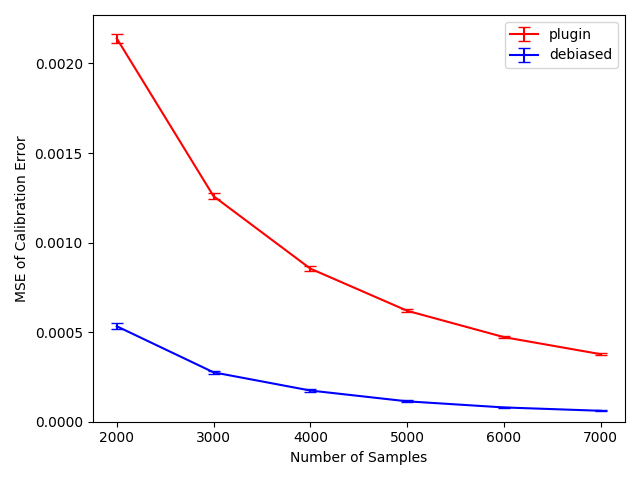}
         \caption{$B = 100$
         }
     \end{subfigure}
  \caption{
    Mean-squared errors of plugin and debiased estimators on a recalibrated VGG16 model on ImageNet with $90\%$ confidence intervals (lower values better). The debiased estimator is closer to the ground truth, which corresponds to $0$ on the vertical axis, and much more so when $B$ is large or $n$ is small.
    Note that this is the MSE of the ECE estimates, not the MSE of the model in Figure~\ref{fig:nan2}.
}
  \label{fig:mse_estimators_imagenet_ece_bins}
\end{figure}

We run multiclass top-label calibration experiment on CIFAR-10 and ImageNet which suggests that the debiased estimator produces better estimates of the calibration error than the plugin estimator. We describe the protocol for ImageNet first, which is similar to the experimental protocol in Section~\ref{sec:verifying_calibration_experiments}. We split the validation set of size 50,000 into two sets $\calset{}$ and $\verifset{}$ of sizes 3,000 and 47,000 respectively. Note that a practitioner would not need so many data points when estimating their model's calibration, we use 47,000 points only so that we can reliably compare the estimators. We use $\calset{}$ to re-calibrate a trained VGG-16 model and select a binning scheme $\bins{}$ so that each bin contains an equal number of points (uniform-mass binning). We calibrate the top probability prediction as described in Section~\ref{sec:formulation} using Platt scaling. For varying values of $n$, we sample $n$ points with replacement from $\verifset{}$, and estimate the binned $\ell_1$ calibration error (ECE) using the plugin estimator and our proposed debiased estimator. We used $B = 100$ or $B = 15$ bins in our experiments. We then compute the squared deviation of these estimates from the binned $\ell_1$ calibration error measured on the entire set $\verifset{}$. We repeat this resampling 1,000 times to get the mean squared deviation of the estimates from the ground truth and 90\% confidence intervals. Figure~\ref{fig:mse_estimators_imagenet_ece_bins} shows that the debiased estimates are much closer to the ground truth than the plugin estimates---the difference is especially significant when the number of samples $n$ is small or the number of bins $B$ is large. Note that having a perfect estimate corresponds to $0$ on the vertical axis.

\begin{figure}
  \centering
  \centering
     \begin{subfigure}[b]{0.45\textwidth}
         \centering
         \includegraphics[width=\textwidth]{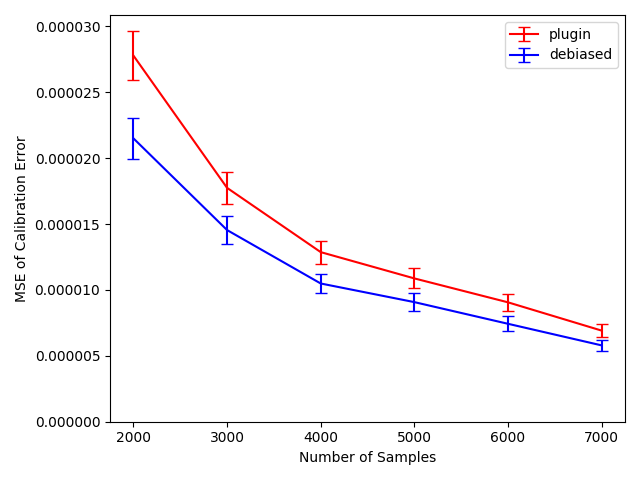}
         \caption{$B = 15$
         }
     \end{subfigure}
     \hfill
     \begin{subfigure}[b]{0.45\textwidth}
         \centering
         \includegraphics[width=\textwidth]{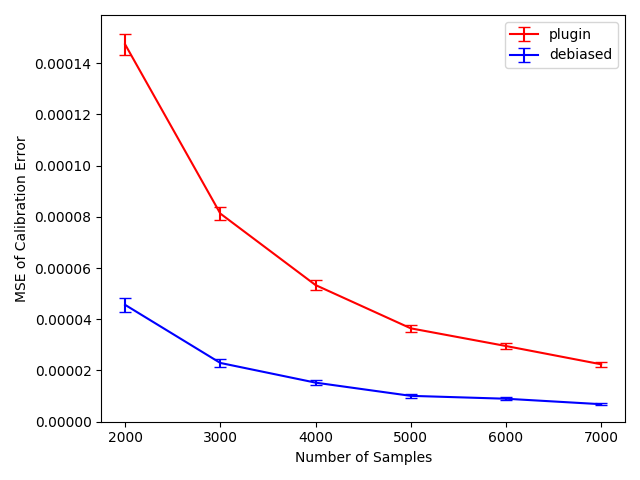}
         \caption{$B = 100$
         }
     \end{subfigure}
  \caption{
    Mean-squared errors of plugin and debiased estimators on a recalibrated VGG16 model on CIFAR-10 with $90\%$ confidence intervals (lower values better). The debiased estimator is closer to the ground truth, which corresponds to $0$ on the vertical axis, especially when $B$ is large or $n$ is small.
    Note that this is the MSE of the ECE estimates, not the MSE of the model in Figure~\ref{fig:nan2}.
}
  \label{fig:mse_estimators_cifar_ece_bins}
\end{figure}

For CIFAR-10 we use the same protocol except we split the validation set of size 10,000 into two sets $\calset{}$ and $\verifset{}$ of sizes 3,000 and 7,000 respectively. Figure~\ref{fig:mse_estimators} shows that the debiased estimates are much closer to the ground truth than the plugin estimates in this case as well.

\subsection{Additional experiments for estimating calibration error}

In Section~\ref{sec:verifying_calibration} we ran an experiment on CIFAR-10 to show that the debiased estimator gives estimates closer to the true calubration error than the plugin estimator. To give more insight into this, Figure~\ref{fig:histograms_estimators_bins} shows a histogram of the absolute difference between the estimates and ground truth for the plugin and debiased estimator, over the 1,000 resamples, when we use $B = 10$ or $B = 100$ bins. For $B = 10$ bins it is not completely clear which estimator is doing better but the debiased estimator avoids very bad estimates. However, when $B = 100$, the debiased estimator produces estimates much closer to the ground truth ($0$ on the x-axis).

\begin{figure}
  \centering
  \centering
     \begin{subfigure}[b]{0.45\textwidth}
         \centering
         \includegraphics[width=\textwidth]{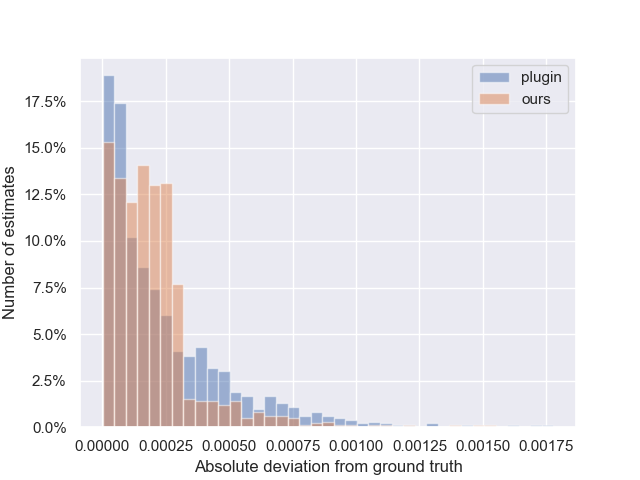}
         \caption{$B = 10$ bins}
     \end{subfigure}
     \hfill
     \begin{subfigure}[b]{0.45\textwidth}
         \centering
         \includegraphics[width=\textwidth]{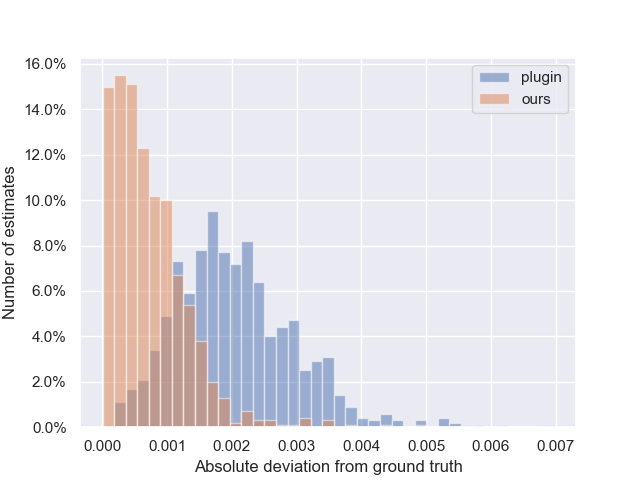}
         \caption{$B = 100$ bins}
     \end{subfigure}
  \caption{Histograms of the absolute value of the difference between estimated and ground truth squared calibration errors ($0$ on the x-axis). For $B = 10$ bins, the results are mixed but we avoid very bad estimates. For $B=100$ our estimates are much closer to ground truth.\tnote{same comments as before}}
  \label{fig:histograms_estimators_bins}
\end{figure}

We also show histograms for the ECE experiments in Appendix~\ref{sec:debiasing_ece_experiments}, in Figure~\ref{fig:histograms_estimators_ece_imagenet_bins} for ImageNet and Figure~\ref{fig:histograms_estimators_ece_cifar_10_bins} for CIFAR-10. These histograms show that the proposed debiased estimator produces estimates much closer to the ground truth than the plugin estimator.

\begin{figure}
  \centering
  \centering
     \begin{subfigure}[b]{0.45\textwidth}
         \centering
         \includegraphics[width=\textwidth]{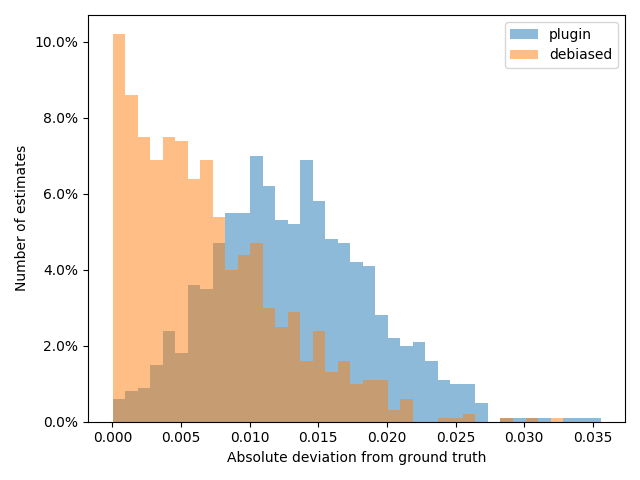}
         \caption{$B = 15$ bins}
     \end{subfigure}
     \hfill
     \begin{subfigure}[b]{0.45\textwidth}
         \centering
         \includegraphics[width=\textwidth]{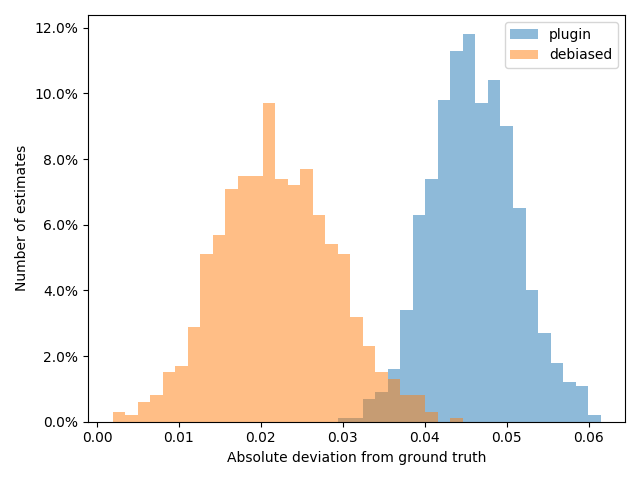}
         \caption{$B = 100$ bins}
     \end{subfigure}
  \caption{Histograms of the absolute value of the difference between estimated and ground truth ECE ($0$ on the x-axis) on ImageNet.}
  \label{fig:histograms_estimators_ece_imagenet_bins}
\end{figure}

\begin{figure}
  \centering
  \centering
     \begin{subfigure}[b]{0.45\textwidth}
         \centering
         \includegraphics[width=\textwidth]{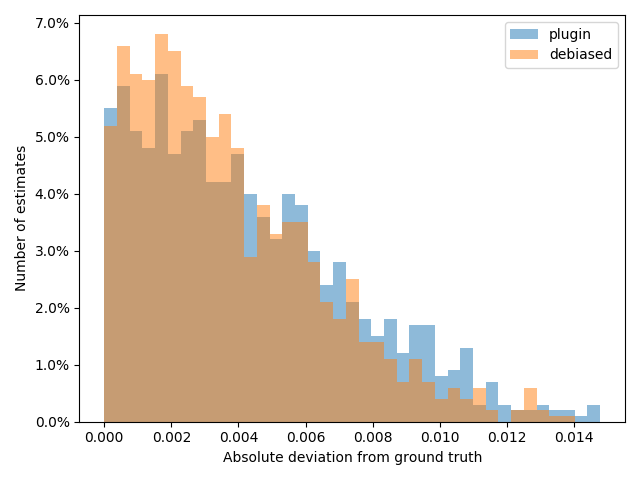}
         \caption{$B = 15$ bins}
     \end{subfigure}
     \hfill
     \begin{subfigure}[b]{0.45\textwidth}
         \centering
         \includegraphics[width=\textwidth]{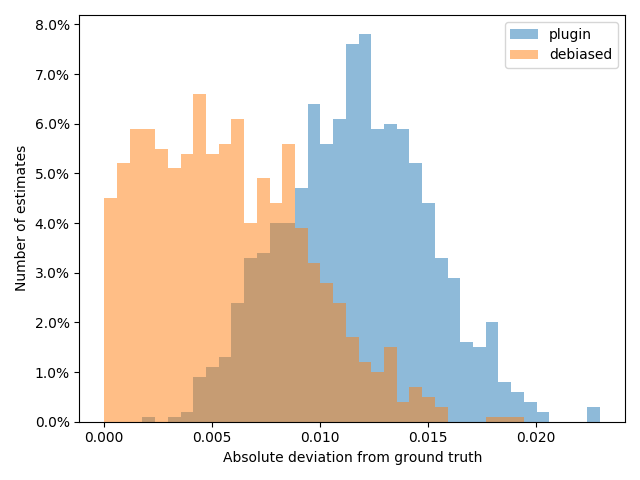}
         \caption{$B = 100$ bins}
     \end{subfigure}
  \caption{Histograms of the absolute value of the difference between estimated and ground truth ECE ($0$ on the x-axis) on CIFAR-10.}
  \label{fig:histograms_estimators_ece_cifar_10_bins}
\end{figure}

We also ran a marginal multiclass calibration experiment on CIFAR-10 to show that our estimator allows us to select models with a lower mean-squared error subject to a given calibration constraint. In this case we split the validation set into $\calset{}$ and $\verifset{}$ of size 6000 and 4000 respectively, and recalibrated a trained model on $\calset{}$. On $\verifset{}$, we estimate the calibration error using the plugin and debiased estimators and use 100 Bootstrap resamples to compute a 90\% upper confidence bound on the estimate (from the variance of the Bootstrap samples). We compute the mean-squared error and the upper bounds on the calibration error for $B = 10, 15, \dots, 100$ and show the Pareto curve in Figure~\ref{fig:mse_vs_ce_estimator}. Figure~\ref{fig:mse_vs_ce_estimator} shows that for any desired calibration error, the debiased estimator enables us to pick out models with a better mean-squared error. For example, if we want a model with calibration error less than $1.5\%$, the debiased estimator tells us we can confidently use 100 bins, while relying on the plugin estimator only lets us use 15 bins and incurs a 13\% higher mean-squared error.



\begin{figure}
  \centering
  \includegraphics[width=0.9\textwidth]{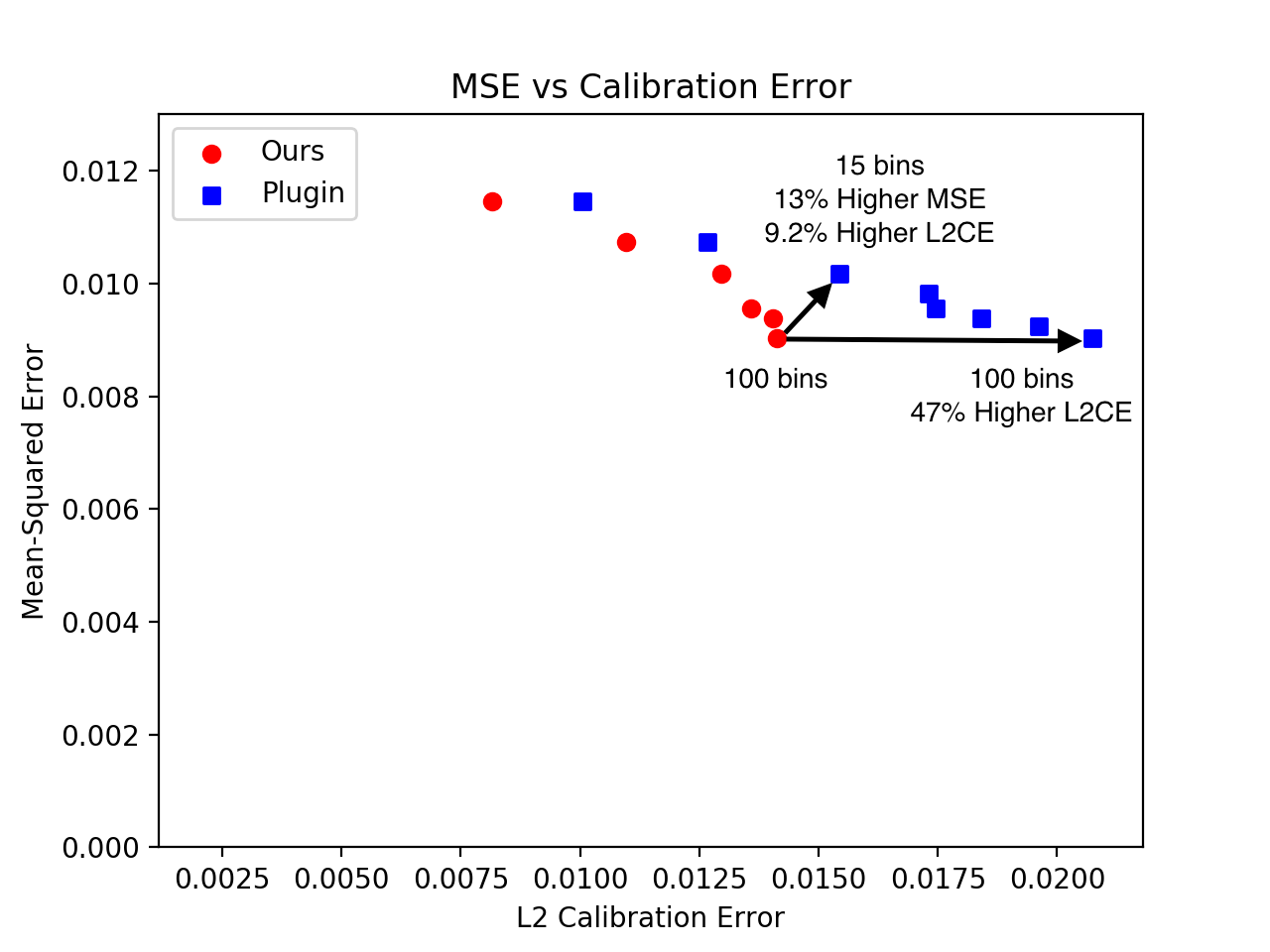}
  \caption{Plot of mean-squared error against 90\% upper bounds on the calibration error computed by the debiased estimator and the plugin estimator, when we vary the number of bins $B$. For a given calibration error, our estimator enables us to choose models with a better mean-squared error. If we want a model with calibration error less than 0.015, the debiased estimator tells us we can confidently use 100 bins, while relying on the plugin estimator only lets us use 15 bins and incurs a 13\% higher mean-squared error.}
  \label{fig:mse_vs_ce_estimator}
\end{figure}

\end{document}